\documentclass[9pt]{sigplanconf}

\usepackage{times}
\usepackage{helvet}
\usepackage{courier}
\usepackage{braket}
\usepackage{mathrsfs}
\usepackage{amsmath, amsthm, amssymb, amsfonts, mathtools}
\usepackage{graphicx, caption, subcaption, tikz}
\usepackage{appendix}
\usepackage{eqproof}
\usepackage{natbib}
\usepackage{tikz}
\usetikzlibrary{decorations.markings}
\usetikzlibrary{positioning, arrows, shapes}

\pgfdeclarelayer{edgelayer}
\pgfdeclarelayer{nodelayer}
\pgfsetlayers{edgelayer,nodelayer,main}

\tikzstyle{none}=[inner sep=0pt]
\tikzstyle{small circ}=[inner sep=0pt, circle, minimum size = 0.2cm,fill=white,draw=black]
\tikzstyle{small_node}=[inner sep=0pt, circle, minimum size = 0.2cm,fill=white,draw=black]
\tikzstyle{blank}=[inner sep=0pt, circle,fill=white,draw=white]
\tikzstyle{box}=[rectangle, minimum size = 0.5cm,fill=white,draw=black]

\usepackage{etoolbox}
\newtoggle{COMMENTS}

\iftoggle{COMMENTS}{
  \newcommand{\margincomment}[2]{\marginpar{\tiny\color{#1}#2}}
  \newcommand{\dm}[1]{\margincomment{green}{#1}}
  \newcommand{\ml}[1]{\margincomment{red}{#1}}
  \newcommand{\bc}[1]{\margincomment{blue}{#1}}
}{
  \newcommand{\dm}[1]{}
  \newcommand{\ml}[1]{}
  \newcommand{\bc}[1]{}
}

\usepackage{stmaryrd}

\DeclareMathOperator{\trace}{tr}
\newcommand{\loewner}{\sqsubseteq}
\newcommand{\satisfies}{\Vdash}
\newcommand{\semantics}[1]{\llbracket #1 \rrbracket}
\newcommand{\lang}[1]{\ensuremath{\textit{#1}}}

\newcommand{\ot}{\otimes}   
\newcommand{\ep}{\varepsilon}

\newcommand {\hypo}{\curlyeqprec} 
\newcommand{\psd}{\succeq} 

\newcommand{\FHilb}{\mathbf{FHilb}}
\newcommand{\CPM}[1]{\mathbf{CPM}(#1)}
\newcommand{\CPMC}{\CPM{\mathcal{C}}}

\newcommand{\freepreg}[1]{\mathsf{Preg}_{#1}}

\newcommand{\functor}[1]{\mathsf{#1}}
\newcommand{\hilbsem}{\functor{Q}}
\newcommand{\cpmsem}{\functor{S}}
\newcommand{\cpmpure}{\functor{E}}

\allowdisplaybreaks
\newtheorem{dfn}{Definition}
\newtheorem{lem}{Lemma}

\newtheorem{thm}{Theorem}
\newtheorem{cor}{Corollary}

\newtheorem {exa}{Example}

\newtheorem*{thm*}{Theorem}
\newtheorem*{exa*}{Example}

\conferenceinfo{}{}
\copyrightyear{}
\copyrightdata{}
\copyrightdoi{}

\titlebanner{banner above paper title}        
\preprintfooter{short description of paper}   

\title{Graded Entailment for Compositional Distributional Semantics}

\authorinfo{Dea~Bankova}
			{Microsoft}
			{d.y.bankova@gmail.com}
\authorinfo{Bob~Coecke \: Martha~Lewis \: Dan~Marsden}
           {University of Oxford}
           {\{coecke, marlew, daniel.marsden\}@cs.ox.ac.uk}

\begin{document}
\maketitle
\begin{abstract}
  The categorical compositional distributional model of natural language provides a conceptually motivated procedure to compute the meaning of sentences, given grammatical structure and the meanings of its words.
  This approach has outperformed other models in mainstream empirical language processing tasks.
  However, until recently it has lacked the crucial feature of lexical entailment -- as do other distributional models of meaning.

  In this paper we solve the problem of entailment for categorical compositional distributional semantics.
  Taking advantage of the abstract categorical framework allows us to vary our choice of model.
  This enables the introduction of a notion of entailment, exploiting ideas from the categorical semantics of partial knowledge in quantum computation.

  The new model of language uses density matrices, on which we introduce a novel robust graded order capturing the \emph{entailment strength} between concepts.
  This graded measure emerges from a general framework for approximate entailment, induced by any commutative monoid.
Quantum logic embeds in our graded order.

Our main theorem  shows that  entailment strength lifts compositionally to the sentence level,  giving a lower bound on sentence entailment. We describe the essential properties of graded entailment such as continuity, and provide a procedure for calculating entailment strength.
\end{abstract}

\category{Artifical Intelligence}{Natural Language Processing}{Lexical Semantics}



\section{Introduction}
Finding a formalization of language in which the meaning of a sentence can be computed from the meaning of its parts has been a long-standing goal in formal and computational linguistics.

Distributional semantics represent individual word meanings as vectors in finite dimensional real vector spaces.
On the other hand, symbolic accounts of meaning combine words via compositional rules to form phrases and sentences.
These two approaches are in some sense orthogonal.
Distributional schemes have no obvious compositional structure, whereas compositional models lack a canonical way of determining the meaning of individual words.
In~\cite{coecke2010}, the authors develop the categorical compositional distributional model of natural language semantics.
This model leverages the shared categorical structure of pregroup grammars and vector spaces to provide a compositional structure for distributional semantics.
It has produced state-of-the-art results in measuring sentence similarity~\cite{Kartsaklis2012, Grefenstette2011},
effectively describing aspects of human understanding of sentences. 

A satisfactory account of natural language should incorporate a suitable notion of lexical entailment.
Until recently, categorical compositional distributional models of meaning have lacked this crucial feature.
In order to address the entailment problem, we exploit the freedom inherent in our abstract categorical framework
to change models. We move from a pure state setting to a category
used to describe mixed states and partial knowledge in the semantics of categorical quantum mechanics.
Meanings are now represented by density~matrices rather than simple vectors.
We use this extra flexibility to capture the concept of hyponymy,
where one word may be seen as an instance of another. For example, \lang{red} is a hyponym of \lang{colour}.
The hyponymy relation can be associated with a notion of logical entailment.
Some entailment is crisp, for example: \lang{dog} entails \lang{animal}. However, we may also wish to permit entailments of differing strengths.
For example, the concept \lang{dog} gives high support to the the concept \lang{pet}, but does not completely entail it: some dogs are working dogs.
The hyponymy~relation we describe here can account for these phenomena.
We should also be able to measure entailment strengths at the sentence level.
For example, we require that \lang{Cujo is a dog} crisply entails \lang{Cujo is an animal},
but that the statement \lang{Cujo is a dog} does not completely entail \lang{Cujo is a pet}.
Again, the relation we describe here will successfully describe this behaviour at the sentence level.

An obvious choice for a logic built upon vector spaces is quantum logic~\cite{birkhoff1936}.
Briefly, this logic represents propositions about quantum systems as projection operators on an appropriate Hilbert space.
These projections form an orthomodular lattice where the distributive law fails in general.
The logical structure is then inherited from the lattice structure in the usual way.
In the current work, we propose an order that embeds the orthomodular lattice of projections, and so contains quantum logic.
This order is based on the L\"owner ordering with propositions represented by density matrices.
When this ordering is applied to density matrices with the standard trace normalization,
no propositions compare, and therefore the L\"owner ordering is useless as applied to density operators.
The trick we use is to develop an approximate entailment relationship which arises naturally from any commutative monoid.
We introduce this in general terms and describe conditions under which this gives a graded measure of entailment.
This grading becomes continuous with respect to noise.
Our framework is flexible enough to subsume the Bayesian partial ordering of~\citet{coecke2011a} and provides it with a grading.

Most closely related to the current work are the ideas in~\cite{balkir2014, balkir2015, balkir2015b}.
In this work, the authors develop a graded form of entailment based on von Neumann entropy and with links to the distributional inclusion hypotheses developed by~\cite{geffet2005}.
The authors show how entailment at the word level carries through to entailment at the sentence level.
However, this is done without taking account of the grading.
In contrast, the measure that we develop here provides a lower bound for the entailment strength between sentences, based on the entailment strength between words.
Further, the measure presented here is applicable to a wider range of sentence types than in~\cite{balkir2015b}. Some of the work presented here was developed here in the first author's MSc thesis \cite{bankova2015}.

Density matrices have also been used in other areas of distributional semantics.
They are exploited in~\citep{Kartsaklis2015, piedeleu2014, piedeleu2015} to encode ambiguity.
\citet{blacoe2013} use density operators to encode the contexts in which a word occurs, but do not use these operators in a compositional structure.

Quantum logic has been applied to distributional semantics in~\cite{widdows2003}, allowing queries of the form `suit NOT lawsuit'.
Here, the vector for `suit' is projected onto the subspace orthogonal to `lawsuit'.
A similar approach, in the field of information retrieval, is described in \citet{van2004}.
In this setting, document retrieval is modelled as a form of quantum logical inference. 

The majorization preordering on density matrices has been extensively used in quantum information \cite{nielsen1999}, however it cannot be turned into a partial order and therefore it is of no use as an entailment relation.

\subsection{Background} 

Within distributional semantics, word meanings are derived from text corpora using word co-occurrence statistics~\cite{Lund1996, Mitchell2010, Bullinaria2007}. Other methods for deriving such meanings may be carried out. In particular, we can view the dimensions of the vector space as attributes of the concept, and experimentally determined attribute importance as the weighting on that dimension as in~\cite{Hampton1987, McRae2005, Vinson2008, Devereux2014}. Distributional models of language have been shown to effectively model various facets of human meaning, such as similarity judgements \cite{mcdonald2001}, word sense discrimination \cite{schutze1998, mccarthy2004} and text comprehension \cite{landauer1997, foltz1998}.

Entailment is  an important and thriving area of research within distributional semantics. The PASCAL Recognising Textual Entailment Challenge~\cite{dagan2006} has attracted a large number of researchers in the area and generated a number of approaches.
Previous lines of research on entailment for distributional semantics investigate the development of
directed similarity measures which can characterize entailment~\cite{weeds2004, kotlerman2010, lenci2012}.
\citet{geffet2005} introduce a pair of \emph{distributional inclusion hypotheses}, where if a word~$v$ entails another word~$w$,
then all the typical features of the word~$v$ will also occur with the word~$w$. Conversely, if all the typical features of~$v$ also occur with~$w$ ,
$v$ is expected to entail $w$. \citet{clarke2009} defines a vector lattice for word vectors, and a notion of graded entailment with the properties of a conditional probability.
\citet{rimell2014} explores the limitations of the distributional inclusion hypothesis by examining the the properties of those features that are not shared between words.
An interesting approach in~\cite{kiela2015} is to incorporate other modes of input into the representation of a word.
Measures of entailment are based on the dispersion of a word representation, together with a similarity measure.

Attempts have also been made to incorporate entailment measures with elements of compositionality.
\citet{baroni2012} exploit the entailment relations between adjective-noun and noun pairs to train a classifier that can detect similar relations.
They further develop a theory of entailment for quantifiers.


\section{Categorical Compositional Distributional Meaning}
\label{sec:DisCo}
Compositional and distributional account of meaning are unified in~\cite{coecke2010}, constructing the meaning of sentences from the meanings of their component parts using their syntactic structure.
\subsection{Pregroup Grammars}
\label{sec:PregroupGrammars}
In order to describe syntactic structure we use Lambek's pregroup grammars \citep{lambek1999}. This choice of grammar is not essential, and other forms of categorial grammar can be used, as argued in \cite{coecke2013}. A pregroup   $(P, \leq, \cdot, 1, (-)^l, (-)^r)$ is a partially ordered monoid $(P, \leq, \cdot, 1)$ where each element $p\in P$ has a left adjoint $p^l$ and a right adjoint $p^r$, such that the following inequalities hold:
\begin{equation}
  \label{eq:preg}
  p^l\cdot p \leq 1 \leq p\cdot p^l \quad \text{ and } \quad p\cdot p^r \leq 1 \leq p^r \cdot p
\end{equation}
Intuitively, we think of the elements of a pregroup as linguistic types. The monoidal structure allows us to form composite types,
and the partial order encodes type reduction. The important right and left adjoints then enable the introduction of types requiring
further elements on either their left or right respectively.

The pregroup grammar~$\freepreg{\mathcal{B}}$ over an alphabet~$\mathcal{B}$ is freely constructed from the atomic types in~$\mathcal{B}$. 
In what follows we use an alphabet $\mathcal{B} = \{n, s\}$. We use the type $s$ to denote a declarative sentence and $n$ to denote a noun. A transitive verb can then be denoted $n^r s n^l$. If a string of words and their types reduces to the type $s$, the sentence is judged grammatical. The sentence \lang{John kicks cats} is typed $n~(n^r s n^l)~ n$, and can be reduced to $s$ as follows: 
\[
n~(n^r s n^l)~ n \leq 1\cdot s n^l n \leq 1 \cdot s \cdot 1 \leq s
\]
This symbolic reduction can also be expressed graphically, as shown in figure~\ref{fig:reduction}.
In this diagrammatic notation, the elimination of types by means of the inequalities~$n \cdot n^r \leq 1$ and~$n^l \cdot n \leq 1$
is denoted by a `cup' while the fact that the type~$s$ is retained is represented by a straight wire.
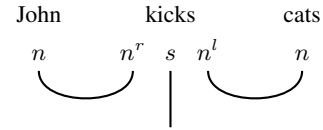
\begin{figure}[htbp]
\centering
\begin{tikzpicture}[text height=1.5 ex]
	\begin{pgfonlayer}{nodelayer}
		\node [style=none] (0) at (-1.75, 3.5) {John};
		\node [style=none] (1) at (0, 3.5) {kicks};
		\node [style=none] (2) at (1.75, 3.5) {cats};
		\node [style=none] (3) at (-1.75, 3) {$n$};
		\node [style=none] (4) at (0, 3) {$s$};
		\node [style=none] (5) at (1.75, 3) {$n$};
		\node [style=none] (6) at (-0.5, 3) {$n^r$};
		\node [style=none] (7) at (0.5, 3) {$n^l$};
		\node [style=none] (8) at (-1.75, 2.75) {};
		\node [style=none] (9) at (-0.5, 2.75) {};
		\node [style=none] (10) at (0, 2.75) {};
		\node [style=none] (11) at (0.5, 2.75) {};
		\node [style=none] (12) at (1.75, 2.75) {};
		\node [style=none] (13) at (0, 2) {};
	\end{pgfonlayer}
	\begin{pgfonlayer}{edgelayer}
		\draw [thick, bend right=90] (8.center) to (9.center);
		\draw [thick, bend right=90] (11.center) to (12.center);
		\draw [thick] (10.center) to (13.center);
	\end{pgfonlayer}
\end{tikzpicture}
\caption{A transitive sentence in the graphical calculus}
\label{fig:reduction}
\end{figure}

\subsection{Compositional Distributional Models}
The symbolic account and distributional approaches are linked by the fact that they share the common structure of a compact closed category.
This compatibility allows the compositional rules of the grammar to be applied in the vector space model. In this way we can map
syntactically well-formed strings of words into one shared meaning space.

A \emph{compact closed category} is a monoidal category in which for each object $A$ there are left and right dual objects $A^l$ and $A^r$, and corresponding
unit and counit morphisms~$\eta^l:I \rightarrow A\otimes A^l$, $\eta^r:I \rightarrow A^r \otimes A$, $\epsilon^l:A^l\otimes A \rightarrow I$, $\epsilon^r:A\otimes A^r \rightarrow I$
such that the following \emph{snake equations} hold:
\[
(1_A\otimes \epsilon^l) \circ (\eta^l \otimes 1_A) = 1_A \qquad (\epsilon^r \otimes 1_{A}) \circ(1_{A} \otimes \eta^r) = 1_A
\]
\[
(\epsilon^l \otimes 1_{A^l}) \circ(1_{A^l} \otimes \eta^l) = 1_{A^l} \qquad (1_{A^r}\otimes \epsilon^r) \circ(\eta^r \otimes 1_{A^r}) = 1_{A^r}
\]

The underlying poset of a pregroup can be viewed as a compact closed category with the monoidal structure given by the pregroup monoid,
and~$\epsilon^l, \eta^l, \eta^r, \epsilon^r$ the unique morphisms witnessing the inequalities of~\eqref{eq:preg}.

Distributional vector space models live in the category $\FHilb$ of finite dimensional real Hilbert spaces and linear maps.
$\FHilb$ is compact closed. Each object $V$ is its own dual and the left and right unit and counit morphisms coincide.
Given a fixed basis $\{ \ket{v_i} \}_i$ of $V$, we define the unit:
\begin{align*}
  \eta : \mathbb{R} &\rightarrow V \otimes V :: 1 \mapsto \sum_i \ket{v_i} \otimes \ket{v_i}
\end{align*}
and counit:
\begin{align*}
  \epsilon: V\otimes V &\rightarrow \mathbb{R} :: \sum_{ij} c_{ij}\ket{v_i} \otimes \ket{v_j} \mapsto  \sum_{i} c_{ii}
\end{align*}
Here we use the physicists bra-ket notation, for details see \cite{Nielsen2010}.

\subsection{Graphical Calculus}
\label{sec:graphcalc}
The morphisms of compact closed categories can be expressed in a convenient graphical calculus~\cite{KellyLaplaza1980} which we
will exploit in the sequel.
Objects are labelled wires, and morphisms are given as vertices with input and output wires.
Composing morphisms consists of connecting input and output wires, and the tensor product is formed by juxtaposition, as
shown in figure~\ref{fig:mongraph}.

\begin{figure}[htbp]
\centering
\begin{tikzpicture}
	\begin{pgfonlayer}{nodelayer}
		\node [style=none] (0) at (-5, 1) {};
		\node [style=none] (1) at (-5, -1) {};
		\node [style=none] (2) at (-3, 1) {};
		\node [style=box] (3) at (-3, 0) {$f$};
		\node [style=none] (4) at (-3, -1) {};
		\node [style=none] (5) at (-0.75, 1) {};
		\node [style=none] (6) at (-0.75, -1) {};
		\node [style=box] (7) at (-0.75, 0) {$g$};
		\node [style=none] (8) at (1.5, -1.25) {};
		\node [style=box] (9) at (1.5, -0.5) {$g$};
		\node [style=none] (10) at (0.25, 1) {};
		\node [style=box] (11) at (0.25, 0) {$f$};
		\node [style=none] (12) at (0.25, -1) {};
		\node [style=none] (13) at (1.5, 1.25) {};
		\node [style=box] (14) at (1.5, 0.5) {$f$};
		\node [style=none] (15) at (-5.5, -1.75) {};
		\node [style=box] (16) at (-5.5, -2.5) {$f$};
		\node [style=none] (17) at (-5.5, -3.25) {};
		\node [style=none] (18) at (-4.5, -1.75) {};
		\node [style=box] (19) at (-4.5, -2.5) {$g$};
		\node [style=none] (20) at (-4.5, -3.25) {};
		\node [style=none] (21) at (-3, -1.75) {};
		\node [style=box] (22) at (-3, -2.5) {$f$};
		\node [style=none] (23) at (-3, -3.25) {};
		\node [style=none] (24) at (-2, -1.75) {};
		\node [style=box] (25) at (-2, -2.5) {$g$};
		\node [style=none] (26) at (-2, -3.25) {};
		\node [style=none] (27) at (-4.5, 0) {$A$};
		\node [style=none] (28) at (-2.75, 0.75) {$A$};
		\node [style=none] (29) at (-2.75, -0.75) {$B$};
		\node [style=none] (30) at (-0.5, 0.75) {$B$};
		\node [style=none] (31) at (-0.5, -0.75) {$C$};
		\node [style=none] (32) at (0.5, 0.75) {$A$};
		\node [style=none] (33) at (0.5, -0.75) {$B$};
		\node [style=none] (34) at (1.75, 1) {$A$};
		\node [style=none] (35) at (1.75, 0) {$B$};
		\node [style=none] (36) at (1.75, -1) {$C$};
		\node [style=none] (37) at (-5.25, -2) {$A$};
		\node [style=none] (38) at (-5.25, -3) {$B$};
		\node [style=none] (39) at (-4.25, -2) {$C$};
		\node [style=none] (40) at (-4.25, -3) {$D$};
		\node [style=none] (41) at (-2.75, -2) {$A$};
		\node [style=none] (42) at (-2.75, -3) {$B$};
		\node [style=none] (43) at (-1.75, -2) {$C$};
		\node [style=none] (44) at (-1.75, -3) {$D$};
		\node [style=none] (45) at (-5, -2.5) {$\otimes$};
		\node [style=none] (46) at (-0.25, 0) {$\circ$};
		\node [style=none] (47) at (1, 0) {$=$};
		\node [style=none] (48) at (-3.75, -2.5) {$=$};
		\node [style=none] (49) at (0, -2.25) {};
		\node [style=none] (50) at (0.5, -2.25) {};
		\node [style=none] (51) at (0.5, -2.75) {};
		\node [style=none] (52) at (-0.25, -2.75) {};
		\node [style=none] (53) at (0.25, -1.75) {};
		\node [style=none] (54) at (0.25, -3.25) {};
		\node [style=none] (55) at (0.25, -2.5) {$f^*$};
		\node [style=none] (56) at (0.25, -2.25) {};
		\node [style=none] (57) at (0.25, -2.75) {};
		\node [style=none] (58) at (1, -2.5) {$=$};
		\node [style=none] (59) at (0.5, -2) {$A$};
		\node [style=none] (60) at (0.5, -3) {$B$};
		\node [style=none] (61) at (1.75, -1.75) {};
		\node [style=none] (62) at (2, -2) {$A$};
		\node [style=none] (63) at (2, -3) {$B$};
		\node [style=none] (64) at (1.5, -2.75) {};
		\node [style=none] (65) at (2, -2.25) {};
		\node [style=none] (66) at (1.5, -2.25) {};
		\node [style=none] (67) at (2.25, -2.75) {};
		\node [style=none] (68) at (1.75, -2.25) {};
		\node [style=none] (69) at (1.75, -3.25) {};
		\node [style=none] (70) at (1.75, -2.75) {};
		\node [style=none] (71) at (1.75, -2.5) {$f$};
	\end{pgfonlayer}
	\begin{pgfonlayer}{edgelayer}
		\draw [thick, >->] (0.center) to (1.center);
		\draw [thick, >-] (2.center) to (3);
		\draw [thick, ->] (3) to (4.center);
		\draw [thick, >-] (5.center) to (7);
		\draw [thick, ->] (7) to (6.center);
		\draw [thick, >-] (10.center) to (11);
		\draw [thick, ->] (11) to (12.center);
		\draw [thick, >-] (13.center) to (14);
		\draw [thick] (14) to (9);
		\draw [thick, ->] (9) to (8.center);
		\draw [thick, >-] (15.center) to (16);
		\draw [thick, ->] (16) to (17.center);
		\draw [thick, >-] (18.center) to (19);
		\draw [thick, ->] (19) to (20.center);
		\draw [thick, >-] (21.center) to (22);
		\draw [thick, ->] (22) to (23.center);
		\draw [thick, >-] (24.center) to (25);
		\draw [thick, ->] (25) to (26.center);
		\draw [thick, >-] (53.center) to (56.center);
		\draw (50.center) to (49.center);
		\draw (49.center) to (52.center);
		\draw (52.center) to (51.center);
		\draw (51.center) to (50.center);
		\draw [thick, ->] (57.center) to (54.center);
		\draw [thick, >-] (61.center) to (68.center);
		\draw (66.center) to (65.center);
		\draw (65.center) to (67.center);
		\draw (67.center) to (64.center);
		\draw (64.center) to (66.center);
		\draw [thick, ->] (70.center) to (69.center);
	\end{pgfonlayer}
\end{tikzpicture}
\caption{Monoidal Graphical Calculus}
\label{fig:mongraph}
\end{figure}
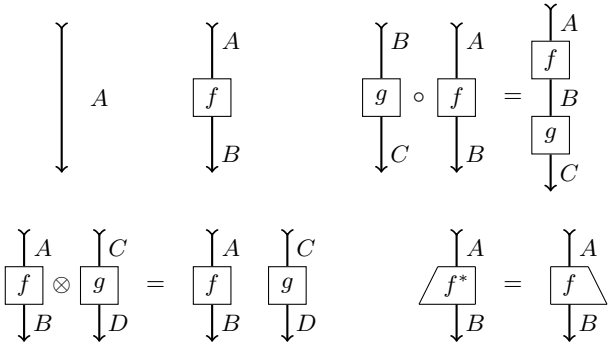
By convention the wire for the monoidal unit is omitted.
The morphisms $\epsilon$ and $\eta$ can then be represented by `cups' and `caps' as shown in figure~\ref{fig:comgraph}.
The snake equations can be seen as straightening wires, as shown in figure~\ref{fig:snake}.

\begin{figure}[htbp]
\centering
\begin{tikzpicture}
	\begin{pgfonlayer}{nodelayer}
		\node [style=none] (0) at (-2.5, 0.5) {$\epsilon^l$};
		\node [style=none] (1) at (-1, 0.5) {};
		\node [style=none] (2) at (2.5, 0.5) {$\epsilon^r$};
		\node [style=none] (3) at (2.5, -0.5) {$\eta^r$};
		\node [style=none] (4) at (-1, -0.5) {};
		\node [style=none] (5) at (1, -0.5) {};
		\node [style=none] (6) at (2, -0.5) {};
		\node [style=none] (7) at (1, 0.5) {};
		\node [style=none] (8) at (-2, 0.5) {};
		\node [style=none] (9) at (-2.5, -0.5) {$\eta^l$};
		\node [style=none] (10) at (2, 0.5) {};
		\node [style=none] (11) at (-2, -0.5) {};
	\end{pgfonlayer}
	\begin{pgfonlayer}{edgelayer}
		\draw [thick, >->, bend left=90, looseness=1.25] (1.center) to (8.center);
		\draw [thick, >->, bend right=90, looseness=1.25] (4.center) to (11.center);
		\draw [thick, >->, bend right=90, looseness=1.25] (7.center) to (10.center);
		\draw [thick, >->, bend left=90, looseness=1.25] (5.center) to (6.center);
	\end{pgfonlayer}
\end{tikzpicture}
\caption{Compact Structure Graphically}
\label{fig:comgraph}
\end{figure}
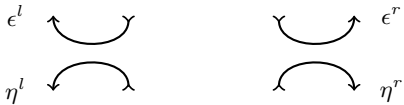

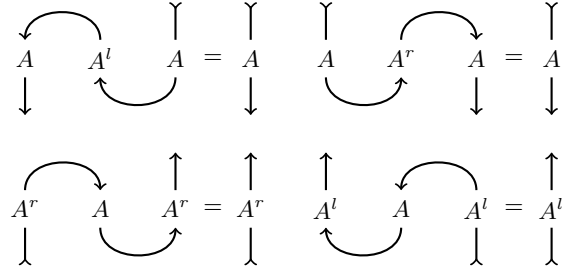
\begin{figure}[htbp]
\centering
\begin{tikzpicture}
	\begin{pgfonlayer}{nodelayer}
		\node [style=none] (0) at (-3.5, 1) {$A$};
		\node [style=none] (1) at (-2.5, 1) {$A^l$};
		\node [style=none] (2) at (-1.5, 1) {$A$};
		\node [style=none] (3) at (-1, 1) {$=$};
		\node [style=none] (4) at (-0.5, 1) {$A$};
		\node [style=none] (5) at (0.5, 1) {$A$};
		\node [style=none] (6) at (1.5, 1) {$A^r$};
		\node [style=none] (7) at (2.5, 1) {$A$};
		\node [style=none] (8) at (3, 1) {$=$};
		\node [style=none] (9) at (3.5, 1) {$A$};
		\node [style=none] (10) at (-3.5, 1.25) {};
		\node [style=none] (11) at (-3.5, 0.75) {};
		\node [style=none] (12) at (-3.5, 0.25) {};
		\node [style=none] (13) at (-2.5, 1.25) {};
		\node [style=none] (14) at (-2.5, 0.75) {};
		\node [style=none] (15) at (-1.5, 0.75) {};
		\node [style=none] (16) at (-1.5, 1.25) {};
		\node [style=none] (17) at (-1.5, 1.75) {};
		\node [style=none] (18) at (-0.5, 1.75) {};
		\node [style=none] (19) at (-0.5, 0.25) {};
		\node [style=none] (20) at (0.5, 1.75) {};
		\node [style=none] (21) at (0.5, 1.25) {};
		\node [style=none] (22) at (0.5, 0.75) {};
		\node [style=none] (23) at (1.5, 0.75) {};
		\node [style=none] (24) at (1.5, 1.25) {};
		\node [style=none] (25) at (2.5, 1.25) {};
		\node [style=none] (26) at (2.5, 0.75) {};
		\node [style=none] (27) at (2.5, 0.25) {};
		\node [style=none] (28) at (3.5, 1.75) {};
		\node [style=none] (29) at (3.5, 0.25) {};
		\node [style=none] (30) at (-0.5, 1.25) {};
		\node [style=none] (31) at (-0.5, 0.75) {};
		\node [style=none] (32) at (3.5, 1.25) {};
		\node [style=none] (33) at (3.5, 0.75) {};
		\node [style=none] (34) at (-0.5, -1.25) {};
		\node [style=none] (35) at (-0.5, -1.75) {};
		\node [style=none] (36) at (1.5, -1.25) {};
		\node [style=none] (37) at (3.5, -1.25) {};
		\node [style=none] (38) at (3.5, -1) {$A^l$};
		\node [style=none] (39) at (2.5, -1) {$A^l$};
		\node [style=none] (40) at (-2.5, -0.75) {};
		\node [style=none] (41) at (0.5, -0.75) {};
		\node [style=none] (42) at (1.5, -0.75) {};
		\node [style=none] (43) at (2.5, -1.75) {};
		\node [style=none] (44) at (-3.5, -1.75) {};
		\node [style=none] (45) at (3.5, -0.25) {};
		\node [style=none] (46) at (-3.5, -0.75) {};
		\node [style=none] (47) at (0.5, -1) {$A^l$};
		\node [style=none] (48) at (-0.5, -0.25) {};
		\node [style=none] (49) at (-1.5, -1) {$A^r$};
		\node [style=none] (50) at (0.5, -0.25) {};
		\node [style=none] (51) at (-3.5, -1.25) {};
		\node [style=none] (52) at (2.5, -1.25) {};
		\node [style=none] (53) at (-1, -1) {$=$};
		\node [style=none] (54) at (1.5, -1) {$A$};
		\node [style=none] (55) at (3.5, -0.75) {};
		\node [style=none] (56) at (3, -1) {$=$};
		\node [style=none] (57) at (-3.5, -1) {$A^r$};
		\node [style=none] (58) at (-1.5, -0.25) {};
		\node [style=none] (59) at (3.5, -1.75) {};
		\node [style=none] (60) at (0.5, -1.25) {};
		\node [style=none] (61) at (-1.5, -1.25) {};
		\node [style=none] (62) at (-0.5, -1) {$A^r$};
		\node [style=none] (63) at (-0.5, -0.75) {};
		\node [style=none] (64) at (2.5, -0.75) {};
		\node [style=none] (65) at (-2.5, -1.25) {};
		\node [style=none] (66) at (-1.5, -0.75) {};
		\node [style=none] (67) at (-2.5, -1) {$A$};
	\end{pgfonlayer}
	\begin{pgfonlayer}{edgelayer}
		\draw [thick, >-] (17.center) to (16.center);
		\draw [thick, ->, bend left=90, looseness=1.25] (15.center) to (14.center);
		\draw [thick, ->, bend right=90, looseness=1.25] (13.center) to (10.center);
		\draw [thick, ->] (11.center) to (12.center);
		\draw [thick, >-] (18.center) to (30.center);
		\draw [thick, ->] (31.center) to (19.center);
		\draw [thick, >-] (20.center) to (21.center);
		\draw [thick, ->, bend right=90, looseness=1.25] (22.center) to (23.center);
		\draw [thick, ->, bend left=90, looseness=1.25] (24.center) to (25.center);
		\draw [thick, ->] (26.center) to (27.center);
		\draw [thick, >-] (28.center) to (32.center);
		\draw [thick, ->] (33.center) to (29.center);
		\draw [thick, <-] (58.center) to (66.center);
		\draw [thick, <-, bend left=90, looseness=1.25] (61.center) to (65.center);
		\draw [thick, <-, bend right=90, looseness=1.25] (40.center) to (46.center);
		\draw [thick, -<] (51.center) to (44.center);
		\draw [thick, <-] (48.center) to (63.center);
		\draw [thick, -<] (34.center) to (35.center);
		\draw [thick, <-] (50.center) to (41.center);
		\draw [thick, <-, bend right=90, looseness=1.25] (60.center) to (36.center);
		\draw [thick, <-, bend left=90, looseness=1.25] (42.center) to (64.center);
		\draw [thick, -<] (52.center) to (43.center);
		\draw [thick, <-] (45.center) to (55.center);
		\draw [thick, -<] (37.center) to (59.center);
	\end{pgfonlayer}
\end{tikzpicture}
\caption{The Snake Equations}
\label{fig:snake}
\end{figure}

\subsection{Grammatical Reductions in Vector Spaces}
Following \cite{preller2011}, reductions of the pregroup grammar  may be mapped into the category $\FHilb$ of finite dimensional Hilbert spaces and linear maps
using an appropriate strong monoidal functor $\hilbsem$:
\[
\hilbsem: \mathbf{Preg} \rightarrow \FHilb
\]
Strong monoidal functors automatically preserve the compact closed structure.
For our example~$\freepreg{\{n,s\}}$, we must map the noun and sentence types to appropriate finite dimensional vector spaces:
\[
\hilbsem(n) = N \qquad \hilbsem(s) = S
\]
Composite types are then constructed functorially using the corresponding structure in $\FHilb$.
Each morphism $\alpha$ in the pregroup is mapped to a linear map interpreting sentences of that grammatical type. Then, given word vectors $\ket{w_i}$ with types $p_i$, and  a type reduction $\alpha: p_1, p_2, ... p_n \rightarrow s$, the meaning of the sentence $w_1 w_2 ... w_n$ is given by:
\[
\ket {w_1 w_2 ... w_n} = \hilbsem(\alpha)(\ket{w_1} \otimes \ket{w_2} \otimes ... \otimes \ket{w_n})
\]
For example, as described in section \ref{sec:PregroupGrammars},
transitive verbs have type $n^r s n^l$,  and can therefore represented in $\FHilb$ as a rank 3 space $N \otimes S \otimes N$.
The transitive sentence \lang{John kicks cats} has type $n (n^r s n^l) n$, which reduces to the sentence type via $\epsilon^r \otimes 1_s \otimes \epsilon^l$.
So if we represent $\ket{\lang{kicks}}$ by:  
\[
\ket{\lang{kicks}} = \sum_{ijk} c_{ijk} \ket{e_i} \otimes \ket{s_j} \otimes \ket{e_k}
\]
using the definitions of the counits in $\FHilb$ we then we have:
\begin{align*}
\ket{\lang{John kicks cats}} &= \epsilon_N \otimes 1_S \otimes \epsilon_N(\ket{\lang{John}} \otimes \ket{\lang{kicks}} \otimes \ket{\lang{cats}})\\
& = \sum_{ijk} c_{ijk} \braket{\lang{John} | e_i} \otimes \ket{s_j} \otimes \braket{e_k | \lang{cats}}\\
& = \sum_j\sum_{ik} c_{ijk} \braket{\lang{John} | e_i} \braket{ e_k| \lang{cats}} \ket{s_j}
\end{align*}
The category $\FHilb$ is actually a $\dagger$-compact~closed~category. A $\dagger$-compact~closed~category is a compact~closed~category with an additional \emph{dagger functor}
that is an identity on objects involution, satisfying natural coherence conditions. In the graphical calculus, the dagger operation ``flips diagrams upside-down''.
In the case of~$\FHilb$ the dagger sends a linear map to its adjoint, and this allows us to reason about inner products in
a general categorical setting.

Meanings of sentences may be compared using the inner product to calculate the cosine distance between vector representations.
So, if sentence~$s$ has vector representation~$\ket{s}$
and sentence~$s'$ has representation~$\ket{s'}$, their degree of synonymy is given by: 
\begin{equation*}
  \frac{\braket{s | s'}}
       {\sqrt{\braket{s | s}\braket{s' | s'}}}
\end{equation*}
The abstract categorical framework we have introduced allows meanings to be interpreted not just in~$\FHilb$, but in any $\dagger$-compact closed category.
We will exploit this freedom when we move to density matrices.
Detailed presentations of the ideas in this section are given in~\citep{coecke2010, preller2011} and an introduction to relevant category theory given in~\citep{coecke2011}.

\section{Density Matrices in Categorical Compositional Distributional Semantics}
\label{sec:dmds}

\subsection{Positive Operators and Density Matrices}
The methods outlined in section \ref{sec:DisCo} can be applied to the richer setting of density~matrices.
Density~matrices are used in quantum~mechanics to express uncertainty about the state of a system.
For unit~vector $\ket{v}$, the projection operator~$\ket{v}\bra{v}$ onto the subspace spanned by~$\ket{v}$ is called a \emph{pure~state}.
Pure states can be thought of as giving sharp, unambiguous information.
In general, density~matrices are given by a convex sum of pure~states, describing a probabilistic mixture.
States that are not pure are referred to as \emph{mixed~states}.
Necessary and sufficient conditions for an operator~$\rho$ to encode such a probabilistic mixture are:
\begin{itemize}
\item $\forall{v} \in V. \braket{v |\rho | v} \geq 0$
\item $\rho$ is self-adjoint.
\item $\rho$ has trace 1.
\end{itemize}
Operators satisfying the first two axioms are called \emph{positive operators}.
The third axiom ensures that the operator represents a convex mixture of pure states.
However, relaxing this condition gives us different choices for normalization, which we will outline in section~\ref{sec:norm}.

In distributional models of meaning, we can consider the meaning of a word $w$ to be given by a collection of unit vectors~$\{\ket{w_i}\}_i$,
where each~$\ket{w_i}$ represents an instance of the concept expressed by the word.
Each~$\ket{w_i}$ is weighted by~$p_i \in [0,1]$, such that $\sum_i p_i = 1$. These weights describe the meaning of~$w$ as a weighted combination of exemplars.
Then the density~operator:
\[
\semantics{w} = \sum_i p_i \ket{w_i}\bra{w_i}
\]
represents the word~$w$. For example a cat is a fairly typical pet, and a tarantula is less typical, so a simple density operator for the word~\lang{pet} might be:
\begin{equation*}
  \semantics{\lang{pet}} =  0.9 \times \ket{\lang{cat}}\bra{\lang{cat}} + 0.1 \times \ket{\lang{tarantula}}\bra{\lang{tarantula}}
\end{equation*}

\subsection{The CPM Construction}
Applying Selinger's CPM construction~\cite{selinger} to~$\FHilb$ produces a new $\dagger$-compact closed category in which the states are positive operators.
This construction has previously been exploited in a linguistic setting in~\cite{Kartsaklis2015, piedeleu2015, balkir2015}.

Throughout this section~$\mathcal{C}$ denotes an arbitrary $\dag$-compact closed category.
\begin{dfn}[Completely positive morphism]
  A $\mathcal{C}$-morphism~$\varphi: A^* \otimes A \rightarrow B^* \otimes B$ is said to be completely positive~\cite{selinger} if there exists~$C \in \mathsf{Ob}(\mathcal{C})$
  and~$k \in \mathcal{C}(C\otimes A, B)$, such that~$\varphi$ can be written in the form:
\[
 (k_* \otimes k) \circ (1_{A^*} \otimes \eta_C \otimes 1_A)
\] 
\end{dfn}
Identity morphisms are completely positive, and completely positive morphisms are closed under composition in~$\mathcal{C}$, leading to the following:
\begin{dfn}
  If~$\mathcal{C}$ is a $\dag$-compact closed category then~$\CPMC$ is a category with the same objects as~$\mathcal{C}$ and its morphisms are the completely positive morphisms.
\end{dfn}
The $\dagger$-compact structure required for interpreting language in our setting lifts to~$\CPM{\mathcal{C}}$:
\begin{thm}
  $\CPM{\mathcal{C}}$ is also a $\dagger$-compact closed category.
  There is a functor:
  \begin{align*}
    \cpmpure : \mathcal{C} &\rightarrow \CPM{\mathcal{C}}\\
    k &\mapsto k_* \otimes  k
  \end{align*}
  This functor preserves the $\dagger$-compact closed structure, and is faithful ``up to a global phase''~\cite{selinger}.   
\end{thm}

\subsection{Diagrammatic calculus for $\CPMC$}
As $\CPMC$ is also a $\dagger$-compact closed category, we can use the graphical calculus described in section \ref{sec:graphcalc}.
By convention, the diagrammatic calculus for~$\CPMC$ is drawn using thick wires.
The corresponding diagrams in~$\mathcal{C}$ are given in table~\ref{tab:dc}.

\begin{table}
\centering
\caption{Table of diagrams in $\CPMC$ and $\mathcal{C}$}
\label{tab:dc} 
\begin{tabular}{ c c}
  $\CPMC$ & $\mathcal{C}$\\
  \hline
  $E(\ep) = \ep_* \otimes \ep$ & $\ep : A^* \otimes A^* \otimes A \otimes A \rightarrow I$ \\
  \begin{tikzpicture}[scale=0.8]
	\begin{pgfonlayer}{nodelayer}
		\node [style=none] (0) at (-1, 0.25) {$A^*$};
		\node [style=none] (1) at (1, 0.25) {$A$};
		\node [style=none] (2) at (1, 0) {};
		\node [style=none] (3) at (-1, 0) {};
	\end{pgfonlayer}
	\begin{pgfonlayer}{edgelayer}
		\draw [ultra thick, bend left=90] (2.center) to (3.center);
	\end{pgfonlayer}
\end{tikzpicture}&\begin{tikzpicture}[scale=0.8]
	\begin{pgfonlayer}{nodelayer}
		\node [style=none] (0) at (-1, 0.25) {$A^*$};
		\node [style=none] (1) at (0, 0.25) {$A^*$};
		\node [style=none] (2) at (1, 0.25) {$A$};
		\node [style=none] (3) at (2, 0.25) {$A$};
		\node [style=none] (4) at (-1, 0) {};
		\node [style=none] (5) at (0, 0) {};
		\node [style=none] (6) at (1, 0) {};
		\node [style=none] (7) at (2, 0) {};
	\end{pgfonlayer}
	\begin{pgfonlayer}{edgelayer}
		\draw [thick, bend left=90] (6.center) to (4.center);
		\draw [thick, bend left=90] (7.center) to (5.center);
	\end{pgfonlayer}
\end{tikzpicture}\\
  \multicolumn{2}{c}{$\ep: \ket{e_i} \otimes \ket{e_j} \otimes \ket{e_k} \otimes \ket{e_l} \mapsto \braket{e_i | e_k}\braket{e_j| e_l}$} \\
&\\
 $E(\eta) = \eta_* \otimes \eta$ &  $\eta : I \rightarrow A \otimes A\otimes A^* \otimes A^* $\\
\begin{tikzpicture}[scale=0.8]
	\begin{pgfonlayer}{nodelayer}
		\node [style=none] (0) at (-1, -0.25) {$A^*$};
		\node [style=none] (1) at (1, -0.25) {$A$};
		\node [style=none] (2) at (1, 0) {};
		\node [style=none] (3) at (-1, 0) {};
	\end{pgfonlayer}
	\begin{pgfonlayer}{edgelayer}
		\draw [ultra thick, bend right=90, looseness=1.25] (2.center) to (3.center);
	\end{pgfonlayer}
\end{tikzpicture}&\begin{tikzpicture}[scale=0.8]
	\begin{pgfonlayer}{nodelayer}
		\node [style=none] (0) at (-1, -0.25) {$A^*$};
		\node [style=none] (1) at (0, -0.25) {$A^*$};
		\node [style=none] (2) at (1, -0.25) {$A$};
		\node [style=none] (3) at (2, -0.25) {$A$};
		\node [style=none] (4) at (-1, 0) {};
		\node [style=none] (5) at (0, 0) {};
		\node [style=none] (6) at (1, 0) {};
		\node [style=none] (7) at (2, 0) {};
	\end{pgfonlayer}
	\begin{pgfonlayer}{edgelayer}
		\draw [thick, bend right=90, looseness=1.25] (6.center) to (4.center);
		\draw [thick, bend right=90, looseness=1.25] (7.center) to (5.center);
	\end{pgfonlayer}
\end{tikzpicture}\\
\multicolumn{2}{c}{$\eta: 1 \mapsto \sum_{ij}\ket{e_i} \otimes \ket{e_j} \otimes \ket{e_i} \otimes \ket{e_j} $}\\
&\\
\begin{tikzpicture}[scale=0.8]
	\begin{pgfonlayer}{nodelayer}
		\node [style=box, ultra thick] (0) at (0.5, 0) {$f_2$};
		\node [style=none] (1) at (0.5, 1) {};
		\node [style=none] (2) at (-0.5, 1) {};
		\node [style=box, ultra thick] (3) at (-0.5, 0) {$f_1$};
		\node [style=none] (4) at (0.5, -1) {};
		\node [style=none] (5) at (-0.5, -1) {};
		\node [style=none] (6) at (-0.5, -1.25) {$A$};
		\node [style=none] (7) at (0.5, -1.25) {$C$};
		\node [style=none] (8) at (-0.5, 1.25) {$B$};
		\node [style=none] (9) at (0.5, 1.25) {$D$};
	\end{pgfonlayer}
	\begin{pgfonlayer}{edgelayer}
		\draw [ultra thick] (5.center) to (3);
		\draw [ultra thick] (4.center) to (0);
		\draw [ultra thick] (3) to (2.center);
		\draw [ultra thick] (0) to (1.center);
	\end{pgfonlayer}
\end{tikzpicture} &\begin{tikzpicture}[scale=0.8]
	\begin{pgfonlayer}{nodelayer}
		\node [style=none] (0) at (-1.25, 1) {};
		\node [style=none] (1) at (-0.75, -0.25) {};
		\node [style=none] (2) at (-0.75, -1) {};
		\node [style=none] (3) at (0.75, 0) {$f_2$};
		\node [style=none] (4) at (0.5, -1) {};
		\node [style=none] (5) at (1, -0.25) {};
		\node [style=none] (6) at (1, 1) {};
		\node [style=none] (7) at (0.25, -0.25) {};
		\node [style=none] (8) at (-1, 0) {$f_1$};
		\node [style=none] (9) at (-0.75, 0.25) {};
		\node [style=none] (10) at (-1.5, -0.25) {};
		\node [style=none] (11) at (1, 0.25) {};
		\node [style=none] (12) at (-1.25, 0.25) {};
		\node [style=none] (13) at (1.25, 0.25) {};
		\node [style=none] (14) at (1, -1) {};
		\node [style=none] (15) at (0.25, 0.25) {};
		\node [style=none] (16) at (-0.75, 1) {};
		\node [style=none] (17) at (-1.25, -1) {};
		\node [style=none] (18) at (0.5, 1) {};
		\node [style=none] (19) at (-0.5, 0.25) {};
		\node [style=none] (20) at (0.5, 0.25) {};
		\node [style=none] (21) at (-0.5, -0.25) {};
		\node [style=none] (22) at (1.25, -0.25) {};
		\node [style=none] (23) at (-1.5, 0.25) {};
		\node [style=none] (24) at (0.5, -0.25) {};
		\node [style=none] (25) at (-1.25, -0.25) {};
		\node [style=none] (26) at (-1.25, -1.25) {$A^*$};
		\node [style=none] (27) at (-0.75, -1.25) {$C^*$};
		\node [style=none] (28) at (0.5, -1.25) {$C$};
		\node [style=none] (29) at (1, -1.25) {$A$};
		\node [style=none] (30) at (-1.25, 1.25) {$B^*$};
		\node [style=none] (31) at (-0.75, 1.25) {$D^*$};
		\node [style=none] (32) at (0.5, 1.25) {$D$};
		\node [style=none] (33) at (1, 1.25) {$B$};
	\end{pgfonlayer}
	\begin{pgfonlayer}{edgelayer}
		\draw [thick] (10.center) to (23.center);
		\draw [thick] (23.center) to (19.center);
		\draw [thick] (19.center) to (21.center);
		\draw [thick] (21.center) to (10.center);
		\draw [thick] (7.center) to (15.center);
		\draw [thick] (15.center) to (13.center);
		\draw [thick] (13.center) to (22.center);
		\draw [thick] (22.center) to (7.center);
		\draw [thick] (17.center) to (25.center);
		\draw [thick, in=-90, out=91, looseness=0.50] (2.center) to (24.center);
		\draw [thick, in=90, out=-90, looseness=0.50] (1.center) to (14.center);
		\draw [thick, in=-90, out=90] (4.center) to (5.center);
		\draw [thick] (12.center) to (0.center);
		\draw [thick, in=-90, out=90, looseness=0.50] (9.center) to (6.center);
		\draw [thick, in=-90, out=90] (11.center) to (18.center);
		\draw [thick, in=-90, out=90, looseness=0.50] (20.center) to (16.center);
	\end{pgfonlayer}
\end{tikzpicture}\\
\multicolumn{2}{c}{$f_1 \ot f_2:A^* \ot C^*\ot C\ot A \rightarrow B^* \ot D^* \ot D \ot B $}\\
  \hline
\end{tabular}
\end{table}

\subsubsection{Sentence Meaning in the category $\CPM{\FHilb}$}
In the vector space model of distributional models of meaning
the transition between syntax and semantics was achieved via a strong monoidal functor~$\hilbsem : \mathbf{Preg} \rightarrow \FHilb$.
Language can be assigned semantics in~$\CPM{\FHilb}$ in an entirely analogous way via a strong monoidal functor:
\begin{equation*}
  \cpmsem: \mathbf{Preg} \rightarrow \CPM{\FHilb}
\end{equation*}
\begin{dfn}
  Let~$w_1,w_2... w_n$ be a string of words with corresponding grammatical types~$t_i$ in~$\mathbf{Preg}_\mathcal{B}$.
  Suppose that the type reduction is given by~$t_1,...t_n \xrightarrow{r} x$ for some~$x \in \mathsf{Ob}(\mathbf{Preg}_\mathcal{B}$.
  Let~$\semantics{w_i}$ be the meaning of word~$w_i$ in~$\CPM{\FHilb}$, i.e. a state of the form $I\rightarrow \cpmsem(t_i)$. Then the meaning of~$w_1 w_2 ... w_n$ is given by:
\[
\semantics{w_1 w_2 ... w_n} = \cpmsem(r)(\semantics{w_1} \otimes ... \otimes \semantics{w_n})
\]
\end{dfn}
We now have all the ingredients to derive sentence meanings in~$\CPM{\FHilb}$.
\begin{exa}\em
We firstly show that the results from $\FHilb$ lift to $\CPM{\FHilb}$.
  Let the noun space~$N$ be a real Hilbert space with basis vectors given by~$\{\ket{n_i}\}_i$, where for some  $i$, $\ket{n_i} = \ket{\lang{Clara}}$ and for some~$j$,
  $\ket{n_j} = \ket{\lang{beer}}$.
  Let the sentence space be another space~$S$ with basis $\{\ket{s_i}\}_i$. The verb $\ket{\lang{likes}}$ is given by:
\[
\ket{\lang{likes}} = \sum_{pqr} C_{pqr} \ket{n_p} \ot \ket{s_q} \ot \ket{n_r}
\]
The density matrices for the nouns \lang{Clara} and \lang{beer} are in fact pure states given by:
\[
\semantics{\lang{Clara}} = \ket{n_i}\bra{n_i} \qquad \text{and} \qquad \semantics{\lang{beer}} = \ket{n_j}\bra{n_j}
\]
and similarly, $\semantics{\lang{likes}}$ in $\CPM{\FHilb}$ is:
\[
\semantics{\lang{likes}} =  \sum_{pqrtuv} C_{pqr}C_{tuv} \ket{n_p}\bra{n_t} \otimes \ket{s_q}\bra{s_u} \otimes \ket{n_r}\bra{n_v}
\]
The meaning of the composite sentence is simply~$(\ep_N \otimes 1_S \otimes \ep_N)$
applied to~$(\semantics{\lang{Clara}} \otimes \semantics{\lang{likes}} \otimes \semantics{\lang{beer}})$ as shown in figure~\ref{fig:ts},
with interpretation in~$\FHilb$ shown in figure~\ref{fig:pure_tscpmc}.
\begin{figure}[htbp]
\centering
\begin{tikzpicture}[scale=0.8]
	\begin{pgfonlayer}{nodelayer}
		\node [style=none] (0) at (-2, 2) {Clara};
		\node [style=none] (1) at (0, 2) {likes};
		\node [style=none] (2) at (2, 2) {beer};
		\node [style=blank] (3) at (0, 0.5) {$S$};
		\node [style=blank] (4) at (2, 0.5) {$N$};
		\node [style=none] (5) at (-2.5, 1) {};
		\node [style=none] (6) at (-1.5, 1) {};
		\node [style=none] (7) at (-2, 1.5) {};
		\node [style=none] (8) at (-0.75, 1) {};
		\node [style=none] (9) at (0.75, 1) {};
		\node [style=none] (10) at (0, 1.5) {};
		\node [style=blank] (11) at (-0.5, 0.5) {$N$};
		\node [style=blank] (12) at (0.5, 0.5) {$N$};
		\node [style=none] (13) at (1.5, 1) {};
		\node [style=none] (14) at (2.5, 1) {};
		\node [style=none] (15) at (2, 1.5) {};
		\node [style=none] (16) at (0, -0.5) {};
		\node [style=blank] (17) at (-2, 0.5) {$N$};
		\node [style=none] (18) at (-2, 1) {};
		\node [style=none] (19) at (0, 1) {};
		\node [style=none] (20) at (2, 1) {};
		\node [style=none] (21) at (0.5, 1) {};
		\node [style=none] (22) at (-0.5, 1) {};
	\end{pgfonlayer}
	\begin{pgfonlayer}{edgelayer}
		\draw [ultra thick] (7.center) to (5.center);
		\draw [ultra thick] (5.center) to (6.center);
		\draw [ultra thick] (6.center) to (7.center);
		\draw [ultra thick] (8.center) to (10.center);
		\draw [ultra thick] (8.center) to (9.center);
		\draw [ultra thick] (9.center) to (10.center);
		\draw [ultra thick] (15.center) to (13.center);
		\draw [ultra thick] (13.center) to (14.center);
		\draw [ultra thick] (14.center) to (15.center);
		\draw [ultra thick, bend right=90, looseness=1.25] (17) to (11);
		\draw [ultra thick] (3) to (16.center);
		\draw [ultra thick, bend left=90, looseness=1.25] (4) to (12);
		\draw [ultra thick] (11) to (22.center);
		\draw [ultra thick] (3) to (19.center);
		\draw [ultra thick] (12) to (21.center);
		\draw [ultra thick] (4) to (20.center);
		\draw [ultra thick] (17) to (18.center);
	\end{pgfonlayer}
\end{tikzpicture}
\caption{A transitive sentence in $\CPMC$}
\label{fig:ts}
\end{figure}
\begin{figure}[htbp]
\centering
\begin{tikzpicture}[scale=0.8]
	\begin{pgfonlayer}{nodelayer}
		\node [style=none] (0) at (0.25, -0.75) {};
		\node [style=none] (1) at (0.75, 0) {};
		\node [style=none] (2) at (1.25, -0.75) {};
		\node [style=none] (3) at (3.25, -0.75) {};
		\node [style=none] (4) at (3.75, 0) {};
		\node [style=none] (5) at (3, 0) {};
		\node [style=none] (6) at (3, 0.5) {};
		\node [style=none] (7) at (3.5, 0.5) {};
		\node [style=none] (8) at (3.25, 0.5) {};
		\node [style=none] (9) at (2.75, 0.5) {};
		\node [style=none] (10) at (2, 0) {};
		\node [style=none] (11) at (2.25, 0.5) {};
		\node [style=none] (12) at (2.5, 0.5) {};
		\node [style=none] (13) at (2.5, -0.75) {};
		\node [style=none] (14) at (2.75, 0) {};
		\node [style=none] (15) at (1.75, 0) {};
		\node [style=none] (16) at (1.5, 0.5) {};
		\node [style=none] (17) at (0, 0.5) {};
		\node [style=none] (18) at (0, 0) {};
		\node [style=blank] (19) at (0.25, -0.5) {$N$};
		\node [style=blank] (20) at (0.75, -0.5) {$S$};
		\node [style=blank] (21) at (1.25, -0.5) {$N'$};
		\node [style=blank] (22) at (2.5, -0.5) {$N'$};
		\node [style=blank] (23) at (3.25, -0.5) {$N'$};
		\node [style=none] (24) at (0.25, 0.5) {};
		\node [style=none] (25) at (0.75, 0.5) {};
		\node [style=none] (26) at (1.25, 0.5) {};
		\node [style=none] (27) at (0.75, -2) {};
		\node [style=none] (28) at (-1, 0) {};
		\node [style=none] (29) at (-3, 0.5) {};
		\node [style=none] (30) at (-1.5, 0.5) {};
		\node [style=blank] (31) at (-0.5, -0.5) {$N$};
		\node [style=blank] (32) at (-2.75, -0.5) {$N$};
		\node [style=none] (33) at (-1.75, 0.5) {};
		\node [style=none] (34) at (-0.5, -0.75) {};
		\node [style=none] (35) at (-3.25, 0) {};
		\node [style=none] (36) at (-2, 0) {};
		\node [style=blank] (37) at (-3.5, -0.5) {$N$};
		\node [style=none] (38) at (-2.75, -0.75) {};
		\node [style=none] (39) at (-2.25, 0) {};
		\node [style=none] (40) at (-4, 0) {};
		\node [style=none] (41) at (-3.5, 0.5) {};
		\node [style=none] (42) at (-0.25, 0) {};
		\node [style=none] (43) at (-2.5, 0.5) {};
		\node [style=none] (44) at (-1, -2) {};
		\node [style=none] (45) at (-3.25, 0.5) {};
		\node [style=blank] (46) at (-1.5, -0.5) {$N'$};
		\node [style=none] (47) at (-1, 0.5) {};
		\node [style=none] (48) at (-0.5, 0.5) {};
		\node [style=none] (49) at (-1.5, -0.75) {};
		\node [style=none] (50) at (-2.75, 0.5) {};
		\node [style=none] (51) at (-3, 0) {};
		\node [style=none] (52) at (-3.75, 0.5) {};
		\node [style=none] (53) at (-0.25, 0.5) {};
		\node [style=blank] (54) at (-1, -0.5) {$S$};
		\node [style=none] (55) at (-3.5, -0.75) {};
		\node [style=none] (56) at (-3.5, 0) {};
		\node [style=none] (57) at (-2.75, 0) {};
		\node [style=none] (58) at (-1.5, 0) {};
		\node [style=none] (59) at (-0.5, 0) {};
		\node [style=none] (60) at (0.25, 0) {};
		\node [style=none] (61) at (1.25, 0) {};
		\node [style=none] (62) at (2.5, 0) {};
		\node [style=none] (63) at (3.25, 0) {};
		\node [style=none] (64) at (-3, 1) {Clara};
		\node [style=none] (65) at (0, 1) {likes};
		\node [style=none] (66) at (2.75, 1) {beer};
	\end{pgfonlayer}
	\begin{pgfonlayer}{edgelayer}
		\draw (52.center) to (40.center);
		\draw (40.center) to (35.center);
		\draw (35.center) to (45.center);
		\draw (45.center) to (52.center);
		\draw (29.center) to (51.center);
		\draw (51.center) to (39.center);
		\draw (39.center) to (43.center);
		\draw (43.center) to (29.center);
		\draw (33.center) to (36.center);
		\draw (36.center) to (42.center);
		\draw (42.center) to (53.center);
		\draw (53.center) to (33.center);
		\draw (17.center) to (18.center);
		\draw (18.center) to (15.center);
		\draw (15.center) to (16.center);
		\draw (16.center) to (17.center);
		\draw (54) to (44.center);
		\draw (20) to (27.center);
		\draw (11.center) to (10.center);
		\draw (9.center) to (14.center);
		\draw (10.center) to (14.center);
		\draw (9.center) to (11.center);
		\draw (6.center) to (5.center);
		\draw (5.center) to (4.center);
		\draw (4.center) to (7.center);
		\draw (7.center) to (6.center);
		\draw (28.center) to (54);
		\draw (1.center) to (20);
		\draw (56.center) to (37);
		\draw (57.center) to (32);
		\draw (58.center) to (46);
		\draw (59.center) to (31);
		\draw (60.center) to (19);
		\draw (61.center) to (21);
		\draw (62.center) to (22);
		\draw (63.center) to (23);
		\draw [bend left=90, looseness=0.75] (3.center) to (2.center);
		\draw [bend left=90, looseness=0.50] (13.center) to (49.center);
		\draw [bend right=90, looseness=0.50] (55.center) to (34.center);
		\draw [bend right=90, looseness=0.75] (38.center) to (0.center);
	\end{pgfonlayer}
\end{tikzpicture}
\caption{A transitive sentence in $\mathcal{C}$ with pure states}
\label{fig:pure_tscpmc}
\end{figure}

In terms of linear algebra, this corresponds to:
\begin{align*}
  \semantics{\lang{Clara likes beer}} &= \varphi(\semantics{\lang{Clara}} \otimes \semantics{\lang{likes}} \otimes \semantics{\lang{beer}})\\
&= \sum_{qu}C_{iqj}C_{iuj}\ket{s_q}\bra{s_u}
\end{align*}
This is a pure state corresponding to the vector $\sum_q C_{iqj} \ket{s_q}$.
\end{exa}
However, in $\CPM{\FHilb}$ we can work with more than the pure states.
\begin{exa}\em
Let the noun space $N$ be a real Hilbert space with basis vectors given by~$\{\ket{n_i}\}_i$. Let:
\begin{gather*}
\ket{\lang{Annie}} = \sum_i a_i\ket{n_i}, \: \ket{\lang{Betty}} = \sum_i b_i\ket{n_i}, \: \ket{\lang{Clara}} = \sum_i c_i\ket{n_i}\\
\ket{\lang{beer}} = \sum_i d_i\ket{n_i}, \quad \ket{\lang{wine}} = \sum_i e_i\ket{n_i}
\end{gather*}
and with the sentence space $S$, we define:
\begin{align*}
\ket{\lang{likes}} &= \sum_{pqr} C_{pqr} \ket{n_p} \ot \ket{s_q} \ot \ket{n_r}\\
\ket{\lang{appreciates}} &= \sum_{pqr} D_{pqr} \ket{n_p} \ot \ket{s_q} \ot \ket{n_r}
\end{align*}
Then, we can set:
\begin{align*}
\semantics{\lang{the sisters}} &= \frac{1}{3}(\ket{\lang{Annie}}\bra{\lang{Annie}} + \ket{\lang{Betty}}\bra{\lang{Betty}} + \ket{\lang{Clara}}\bra{\lang{Clara}})\\
\semantics{\lang{drinks}} &= \frac{1}{2}(\ket{\lang{beer}}\bra{\lang{beer}} + \ket{\lang{wine}}\bra{\lang{wine}})\\
\semantics{\lang{enjoy}} &= \frac{1}{2}(\ket{\lang{like}}\bra{\lang{like}} + \ket{\lang{appreciate}}\bra{\lang{appreciate}})
\end{align*}
Then, the meaning of the sentence:
\[
s = \lang{The sisters enjoy drinks}
\]
is given by:
\[
\semantics{s} = (\ep_N \otimes 1_S \otimes \ep_N)(\semantics{\lang{the sisters}} \otimes \semantics{\lang{enjoy}} \otimes \semantics{\lang{drinks}})
\]
Diagrammatically, this is shown in figure \ref{fig:mixed_tscpmc}.

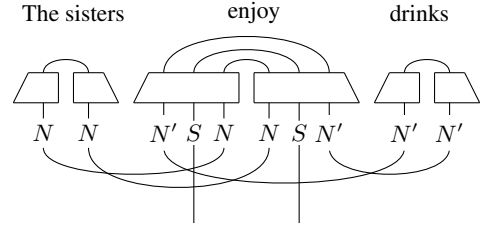
\begin{figure}[htbp]
\centering
\begin{tikzpicture}[scale=0.8]
	\begin{pgfonlayer}{nodelayer}
		\node [style=none] (0) at (0.25, -0.75) {};
		\node [style=none] (1) at (0.75, 0) {};
		\node [style=none] (2) at (1.25, -0.75) {};
		\node [style=none] (3) at (3.25, -0.75) {};
		\node [style=none] (4) at (3.75, 0) {};
		\node [style=none] (5) at (3, 0) {};
		\node [style=none] (6) at (3, 0.5) {};
		\node [style=none] (7) at (3.5, 0.5) {};
		\node [style=none] (8) at (3.25, 0.5) {};
		\node [style=none] (9) at (2.75, 0.5) {};
		\node [style=none] (10) at (2, 0) {};
		\node [style=none] (11) at (2.25, 0.5) {};
		\node [style=none] (12) at (2.5, 0.5) {};
		\node [style=none] (13) at (2.5, -0.75) {};
		\node [style=none] (14) at (2.75, 0) {};
		\node [style=none] (15) at (1.75, 0) {};
		\node [style=none] (16) at (1.5, 0.5) {};
		\node [style=none] (17) at (0, 0.5) {};
		\node [style=none] (18) at (0, 0) {};
		\node [style=blank] (19) at (0.25, -0.5) {$N$};
		\node [style=blank] (20) at (0.75, -0.5) {$S$};
		\node [style=blank] (21) at (1.25, -0.5) {$N'$};
		\node [style=blank] (22) at (2.5, -0.5) {$N'$};
		\node [style=blank] (23) at (3.25, -0.5) {$N'$};
		\node [style=none] (24) at (0.25, 0.5) {};
		\node [style=none] (25) at (0.75, 0.5) {};
		\node [style=none] (26) at (1.25, 0.5) {};
		\node [style=none] (27) at (0.75, -2) {};
		\node [style=none] (28) at (-1, 0) {};
		\node [style=none] (29) at (-3, 0.5) {};
		\node [style=none] (30) at (-1.5, 0.5) {};
		\node [style=blank] (31) at (-0.5, -0.5) {$N$};
		\node [style=blank] (32) at (-2.75, -0.5) {$N$};
		\node [style=none] (33) at (-1.75, 0.5) {};
		\node [style=none] (34) at (-0.5, -0.75) {};
		\node [style=none] (35) at (-3.25, 0) {};
		\node [style=none] (36) at (-2, 0) {};
		\node [style=blank] (37) at (-3.5, -0.5) {$N$};
		\node [style=none] (38) at (-2.75, -0.75) {};
		\node [style=none] (39) at (-2.25, 0) {};
		\node [style=none] (40) at (-4, 0) {};
		\node [style=none] (41) at (-3.5, 0.5) {};
		\node [style=none] (42) at (-0.25, 0) {};
		\node [style=none] (43) at (-2.5, 0.5) {};
		\node [style=none] (44) at (-1, -2) {};
		\node [style=none] (45) at (-3.25, 0.5) {};
		\node [style=blank] (46) at (-1.5, -0.5) {$N'$};
		\node [style=none] (47) at (-1, 0.5) {};
		\node [style=none] (48) at (-0.5, 0.5) {};
		\node [style=none] (49) at (-1.5, -0.75) {};
		\node [style=none] (50) at (-2.75, 0.5) {};
		\node [style=none] (51) at (-3, 0) {};
		\node [style=none] (52) at (-3.75, 0.5) {};
		\node [style=none] (53) at (-0.25, 0.5) {};
		\node [style=blank] (54) at (-1, -0.5) {$S$};
		\node [style=none] (55) at (-3.5, -0.75) {};
		\node [style=none] (56) at (-3.5, 0) {};
		\node [style=none] (57) at (-2.75, 0) {};
		\node [style=none] (58) at (-1.5, 0) {};
		\node [style=none] (59) at (-0.5, 0) {};
		\node [style=none] (60) at (0.25, 0) {};
		\node [style=none] (61) at (1.25, 0) {};
		\node [style=none] (62) at (2.5, 0) {};
		\node [style=none] (63) at (3.25, 0) {};
		\node [style=none] (64) at (-3, 1.5) {The sisters};
		\node [style=none] (65) at (0, 1.5) {enjoy};
		\node [style=none] (66) at (2.75, 1.5) {drinks};
	\end{pgfonlayer}
	\begin{pgfonlayer}{edgelayer}
		\draw (52.center) to (40.center);
		\draw (40.center) to (35.center);
		\draw (35.center) to (45.center);
		\draw (45.center) to (52.center);
		\draw (29.center) to (51.center);
		\draw (51.center) to (39.center);
		\draw (39.center) to (43.center);
		\draw (43.center) to (29.center);
		\draw (33.center) to (36.center);
		\draw (36.center) to (42.center);
		\draw (42.center) to (53.center);
		\draw (53.center) to (33.center);
		\draw (17.center) to (18.center);
		\draw (18.center) to (15.center);
		\draw (15.center) to (16.center);
		\draw (16.center) to (17.center);
		\draw (54) to (44.center);
		\draw (20) to (27.center);
		\draw (11.center) to (10.center);
		\draw (9.center) to (14.center);
		\draw (10.center) to (14.center);
		\draw (9.center) to (11.center);
		\draw (6.center) to (5.center);
		\draw (5.center) to (4.center);
		\draw (4.center) to (7.center);
		\draw (7.center) to (6.center);
		\draw (28.center) to (54);
		\draw (1.center) to (20);
		\draw [bend left=90] (41.center) to (50.center);
		\draw [bend left=90] (48.center) to (24.center);
		\draw [bend left=90, looseness=0.75] (47.center) to (25.center);
		\draw [bend left=90, looseness=0.75] (30.center) to (26.center);
		\draw [bend right=90] (8.center) to (12.center);
		\draw (56.center) to (37);
		\draw (57.center) to (32);
		\draw (58.center) to (46);
		\draw (59.center) to (31);
		\draw (60.center) to (19);
		\draw (61.center) to (21);
		\draw (62.center) to (22);
		\draw (63.center) to (23);
		\draw [bend left=90, looseness=0.75] (3.center) to (2.center);
		\draw [bend left=90, looseness=0.50] (13.center) to (49.center);
		\draw [bend right=90, looseness=0.50] (55.center) to (34.center);
		\draw [bend right=90, looseness=0.75] (38.center) to (0.center);
	\end{pgfonlayer}
\end{tikzpicture}
\caption{A transitive sentence in $\mathcal{C}$ with impure states}
\label{fig:mixed_tscpmc}
\end{figure}
The impurity is indicated by the fact that the pairs of states are connected by wires \cite{selinger}.
\end{exa}

\section{Predicates and Entailment}
If we consider a model of (non-deterministic) classical computation, a state of a set~$X$ is just a subset~$\rho \subseteq X$. Similarly, a predicate is
a subset~$A \subseteq X$. We say that~$\rho$ satisfies~$A$ if:
\begin{equation*}
  \rho \subseteq A
\end{equation*}
which we write as~$\rho \satisfies A$. Predicate~$A$ entails predicate~$B$, written~$A \models B$ if for every state~$\rho$:
\begin{equation*}
  \rho \satisfies A \quad\Rightarrow\quad \rho \satisfies B
\end{equation*}
Clearly this is equivalent to requiring~$A \subseteq B$.

\subsection{The L\"owner Order}
As our linguistic models derive from a quantum mechanical formalism, positive operators form a natural analogue for
subsets as our predicates. This follows ideas in~\citep{DHondt2006} and earlier work in a probabilistic setting in~\citep{Kozen1983}.
Crucially, we can order positive operators~\cite{Lowner1934}.
\begin{dfn}[L\"owner Order] 
\label{dfn:lowner}
For positive operators~$A$ and~$B$, we define:
\begin{equation*}
  A \loewner B \iff  B -  A \text{ is positive}
\end{equation*}
\end{dfn}
If we consider this as an entailment relationship, we can follow our intuitions from the non-deterministic setting.
Firstly we introduce a suitable notion of satisfaction. For positive operator $A$ and density matrix $\rho$, we define $\rho \satisfies A$ as the positive real number $\trace(\rho A)$.
%
This generalizes satisfaction from a binary relation to a binary function into the positive reals.
We then find that the L\"owner order can equivalently be phrased in terms of satisfaction as follows:
\begin{lem}[\citet{DHondt2006}]
  Let $A$ and $B$ be positive operators. $A \loewner B$ if and only if for all density operators~$\rho$:
  \begin{equation*}
    \rho \satisfies A \quad\leq\quad \rho \satisfies B
  \end{equation*}
\end{lem}
Linguistically, we can interpret this condition as saying that every noun, for example, satisfies predicate~$B$ at least
as strongly as it satisfies predicate~$A$.

\subsection{Quantum Logic}
Quantum logic \citep{birkhoff1936} views the projection operators on a Hilbert space as propositions about
a quantum system. As the L\"owner order restricts to the usual ordering on projection operators, we can
embed quantum logic within the poset of projection operators, providing a direct link to existing theory.

\subsection{A General Setting for Approximate Entailment}
\label{sec:general}
We can build an entailment preorder on any commutative monoid, viewing the underlying set as a collection of propositions.
We then write:
\begin{equation*}
  A \models B
\end{equation*}
and say~$A$ entails~$B$ if there exists a proposition~$D$ such that:
\begin{equation*}
  A + D = B
\end{equation*}
If our commutative monoid is the powerset of some set~$X$, with union the binary operation and unit the empty set, then we recover
our non-deterministic computation example from the previous section. If on the other hand we take our commutative monoid to be the positive operators
on some Hilbert space, with addition of operators and the zero operator as the monoid structure, we recover the L\"owner ordering.

In linguistics, we may ask ourselves does~\lang{dog} entail~\lang{pet}? Na\"ively, the answer is clearly no, not every dog is a pet. This seems
too crude for realistic applications though, most dogs are pets, and so we might say \lang{dog} entails \lang{pet} to some extent. This
motivates our need for an approximate notion of entailment.

For proposition~$E$, we say that $A$ entails~$B$ to the extent~$E$ if:
\begin{equation*}
  A \models B + E
\end{equation*}
We think of~$E$ as a error term, for instance in our dogs and pets example, $E$ adds back in dogs that are not pets.
Expanding definitions, we find~$A$ entails~$B$ to extent~$E$ if there exists~$D$ such that:
\begin{equation}
  \label{eq:symmetricalformulation}
  A + D = B + E
\end{equation}
From this more symmetrical formulation it is easy to see that for arbitrary propositions~$A$, $B$, proposition~$A$ trivially entails~$B$ to extent~$A$, as by commutativity:
\begin{equation*}
  A + B = B + A
\end{equation*}
It is therefore clear that the mere existence of a suitable error term is not sufficient for a weakened notion of entailment. If we restrict our
attention to errors in a complete meet semilattice~$\mathcal{E}_{A,B}$, we can take the lower bound on the~$E$ satisfying equation~\eqref{eq:symmetricalformulation}
as our canonical choice. Finally, if we wish to be able to compare entailment strengths globally, this can be achieved by choosing a
partial order~$\mathcal{K}$ of ``error sizes'' and monotone functions:
\begin{equation*}
  \mathcal{E}_{A,B} \xrightarrow{\kappa_{A,B}} \mathcal{K}
\end{equation*}
sending errors to their corresponding size.

For example, if~$A$ and~$B$ are positive operators, we take our complete lattice of error terms~$\mathcal{E}_{A,B}$ to be all operators
of the form~$(1-k)A$ for~$k \in [0,1]$, ordered by the size of~$1 - k$. We then take~$k$ as the strength of the entailment, and
refer to it as \mbox{\emph{k-hyponymy}}.

In the case of finite sets~$A$, $B$, we take~$\mathcal{E}_{A,B} = \mathcal{P}(A)$, and take the size of the error terms as:
\begin{equation*}
  \frac{\text{cardinality of } E}{\text{cardinality of } A}
\end{equation*}
measuring ``how much'' of~$A$ we have to supplement~$B$ with, as indicated in the shaded region below:
\begin{equation*}
  \begin{tikzpicture}[scale=0.5]
    \begin{scope}
      \fill[black!50] (0.9,-0.9) circle (1.2);
    \end{scope}
    \draw[fill=white] (0,0) circle (2);
    \draw (0.9,-0.9) circle (1.2);
    \node at (-1,1) {$B$};
    \node at (0.9,-0.9) {$A$};
  \end{tikzpicture}
\end{equation*}
In terms of conditional probability, the error size is then:
\begin{equation*}
  P(A \mid \neg B)
\end{equation*}

\subsubsection{$k$-hyponymy Versus General Error Terms}
We can see that the general error terms are strictly more general than considering
the $k$-hyponymy case. If we consider positive operators with matrix representations:
\begin{equation*}
  A = \begin{pmatrix}
    1 & 0 & 0\\
    0 & 1 & 0\\
    0 & 0 & 0
  \end{pmatrix}
  \qquad
  B =
  \begin{pmatrix}
    1 & 0 & 0\\
    0 & 0 & 0\\
    0 & 0 & 1
  \end{pmatrix}
\end{equation*}
Predicate $A$ cannot entail $B$ with any positive strength $k$. We can see $B - kA$ is never a positive operator
as the following expression is always negative:
\begin{equation*}
  \begin{pmatrix}
    0 & 1 & 0
  \end{pmatrix}
  \begin{pmatrix}
    1 - k & 0 & 0\\
    0 & -k & 0\\
    0 & 0 & 1
  \end{pmatrix}
  \begin{pmatrix}
    0 \\ 1 \\ 0
  \end{pmatrix}
\end{equation*}
We can find a more general positive operator $E$ such that $A \loewner B + E$ though, as:
\begin{equation*}
  \begin{pmatrix}
    1 & 0 & 0\\
    0 & 1 & 0\\
    0 & 0 & 0
  \end{pmatrix}
  +
  \begin{pmatrix}
    0 & 0 & 0\\
    0 & 0 & 0\\
    0 & 0 & 1
  \end{pmatrix}
  =
  \begin{pmatrix}
    1 & 0 & 0\\
    0 & 0 & 0\\
    0 & 0 & 1
  \end{pmatrix}
  +
  \begin{pmatrix}
    0 & 0 & 0\\
    0 & 1 & 0\\
    0 & 0 & 0
  \end{pmatrix}
\end{equation*}
Therefore the general error terms offer strictly more freedom than $k$-hyponymy.

\section{Hyponymy in Categorical Compositional Distributional Semantics}

Modelling hyponymy in the categorical compositional distributional semantics framework was first considered in~\cite{balkir2014}.
She introduced an asymmetric similarity measure called \emph{representativeness} on density matrices based on quantum relative entropy. This can be used to translate hyponym-hypernym relations to the level of positive transitive sentences. Our aim here will be to provide an alternative measure which relies only on the properties of density matrices and the fact that they are the states in~$\CPM{\FHilb}$. This will enable us to quantify the \emph{strength} of the hyponymy relationship, described as $k$-hyponymy. The measure of hyponymy that we use has two advantages over the representativeness measure. Firstly, it combines with the linear sentence maps so that we can work with sentence-level entailment across a larger range of sentences. Secondly, due to the way it combines with linear maps, we can give a quantitative measure to sentence-level entailment based on the entailment strengths between words, whereas representativeness is not shown to combine in this way.

\dm{Obvious reviewer question: ``what is better about $k$-hyponymy versus Esma's stuff?'' This is not explained forcefully enough here}
\ml{added sentence to strengthen}

\subsection{Properties of hyponymy} 
Before proceeding with defining the concept of \emph{$k$-hyponymy}, we will list a couple of properties of hyponymy. We will show later that these can be captured by our new measure.

\begin{itemize}
\item \textbf{Asymmetry.} If A is a hyponym of B, then this does not imply that Y is a hyponym of X. In fact, we may even assume that only one of these relationships is possible, and that they are mutually exclusive. For example, \lang{football} is a type of \lang{sport} and hence {\lang{football}-\lang{sport}} is a hyponym-hypernym pair.
  However, \lang{sport} is not a type of \lang{football}. 
\item \textbf{Pseudo-transitivity.} If X is a hyponym of Y and Y is a hyponym of Z, then X is a hyponym of Z.
  However, if the hyponymy is not perfect, then we get a weakened form of transitivity.
  For example, \lang{dog} is a hyponym of \lang{pet}, and \lang{pet} is a hyponym of \lang{things that are cared for}.
  However, not every dog is well cared-for, so the transitivity weakens.
  An outstanding question is where the entailment strength reverses. For example, \lang{dog} imperfectly entails
  \lang{pet}, and \lang{pet} imperfectly entails \lang{mammal}, but \lang{dog} perfectly entails \lang{mammal}.
\end{itemize}

The measure of hyponymy that we described above and named $k$-hyponymy will be defined in terms of density matrices - the containers for word meanings. The idea is then to define a quantitative order on the density matrices, which is not a partial order, but does give us an indication of the asymmetric relationship between words. 

\subsection{Ordering Positive Matrices}
\label{sec:opm}
A density matrix can be used to encode the extent of precision that is needed when describing an action. In the sentence \lang{I took my pet to the vet},
we do not know whether the pet is a dog, cat, tarantula and so on. The sentence \lang{I took my dog to the vet} is more specific. We can think of the meaning of the word
\lang{pet} as represented by:
\begin{align*}
  \semantics{\lang{pet}} = & p_d\ket{\lang{dog}}\bra{\lang{dog}} + p_c\ket{\lang{cat}}\bra{\lang{cat}} +\\
  &\quad  p_t\ket{\lang{tarantula}}\bra{\lang{tarantula}} + ...\\\
  &\mbox{where}\quad \forall i. p_i \geq 0 \quad\mbox{and}\quad \sum_i p_i = 1
\end{align*}
We then wish to develop an order on density matrices so that \lang{dog}, as represented by $\ket{\lang{dog}}\bra{\lang{dog}}$
is more specific than \lang{pet} as represented by $\semantics{\lang{pet}}$.
This ordering may then be viewed as an entailment relation, and we wish to show that entailment between words can lift to the level of sentences,
so that the sentence \lang{I took my dog to the vet} entails the sentence \lang{I took my pet to the vet}.

We now define our notion of approximate entailment, following the discussions of section \ref{sec:general}:
\begin{dfn}[$k$-hyponym]
We say that $A$ is a \emph{$k$-hyponym} of $B$ for a given value of \emph{$k$} in the range $(0,1]$ and write $ A  \hypo_k B $ if:
  \[
  0 \loewner B - kA
  \]
\end{dfn}
Note that such a \emph{$k$} need not be unique or even exist at all. We will consider the interpretation and implications of this later on.  Moreover, whenever we do have \emph{$k$}-hyponymy between A and B, there is necessarily a largest such \emph{$k$}. 
\begin{dfn}[$k$-max hyponym]
If $A$ is a $k$-hyponym of $B$ for any $k \in (0,1]$, then there is necessarily a  maximal possible such $k$. We denote it by $k_{max}$ and define it to be the maximum value of $k$ in the range $(0,1]$ for which we have $A \hypo_k B$, in the sense that there does not exist $k' \in (0,1]$ s.t. $k' > k$ and $A \hypo_{k'} B$. 
\end{dfn}

In general, we are interested in the maximal value $k$ for which $k$-hyponymy holds between two positive operators. This $k$-max value quantifies the strength of the entailment between the two operators.
\dm{The next part should be emphasized more, we have a usual theorem here, but we spend much more discussion on the normalization stuff where the benefits are a lot less clear}
In what follows, for operator $A$ we write $A^+$ for the corresponding Moore-Penrose pseudo-inverse and $supp(A)$ for the support of $A$.
\begin{lem}[\citet{balkir2014}]
  \label{lem:supp}
  Let~$A,B$ be positive operators.
  \begin{equation*}
    supp(A) \subseteq supp(B) \iff \exists k. k > 0 \mbox{ and } B - kA \geq 0
  \end{equation*}
\end{lem}
\begin{lem}
  For positive self-adjoint matrices~$A$, $B$ such that:
  \begin{equation*}
    supp(A) \subseteq supp(B)
  \end{equation*}
  $B^+A$ has non-negative eigenvalues.
\end{lem}
We now develop an expression for the optimal~$k$ in terms of the matrices~$A$ and~$B$.
\begin{thm}
\label{thm:kmax}
For positive self-adjoint matrices~$A$, $B$ such that:
\begin{equation*}
  supp(A) \subseteq supp(B)
\end{equation*}
the maximum~$k$ such that~$B - kA \geq 0$ is given by~$1/\lambda$ where~$\lambda$ is the maximum eigenvalue of~$B^+A$.
\end{thm}

\subsection{Properties of $k$-hyponymy}
\paragraph{Reflexivity}
$k$-hyponymy is reflexive for $k$ = 1. For any operator $A$, $A - A = 0$.
\paragraph{Symmetry}
$k$-hyponymy is neither symmetric nor anti-symmetric. For example it is not symmetric since:
\[
  \begin{pmatrix}
    1 & 0\\
    0 & 0\\
  \end{pmatrix}
\hypo_1
  \begin{pmatrix}
    1 & 0\\
    0 & 1\\
  \end{pmatrix}
\text{ but } 
  \begin{pmatrix}
    1 & 0\\
    0 & 1\\
  \end{pmatrix}
\not \hypo_k
  \begin{pmatrix}
    1 & 0\\
    0 & 0\\
  \end{pmatrix}
\]
but it is also not anti-symmetric, since

\[
  \begin{pmatrix}
    1 & 0\\
    0 & 1/2\\
  \end{pmatrix}
\hypo_{1/2}
  \begin{pmatrix}
    1/2 & 0\\
    0 & 1\\
  \end{pmatrix}
\text{ and } 
  \begin{pmatrix}
    1/2 & 0\\
    0 & 1\\
  \end{pmatrix}
\hypo_{1/2}
  \begin{pmatrix}
    1 & 0\\
    0 & 1/2\\
  \end{pmatrix}
\]

\paragraph{Transitivity}
$k$-hyponymy satisfies a version of transitivity.
Suppose $A \hypo_k B$ and $B \hypo_l C$. Then $A \hypo_{kl} C$, since:
\[
B \loewner kA \text{ and } C \loewner lB \implies C \loewner klA
\]
by transitivity of the L\"owner order.

For the maximal values $k_\lang{max}$, $l_\lang{max}$, $m_\lang{max}$ such that $A \hypo_{k_\lang{max}} B$, $B \hypo_{l_\lang{max}} C$ and $A \hypo_{m_\lang{max}} C$, we have the inequality
\[
m_\lang{max} \geq k_\lang{max} l_\lang{max}
\]

\paragraph{Continuity}
For $A \hypo_k B$, when there is a small perturbation to $A$, there is a correspondingly small decrease in the value of $k$. The perturbation must lie in the support of $B$, but can introduce off-diagonal elements.

\begin{thm}
\label{thm:continuity}
Given $A \hypo_k B$ and density operator $\rho$ such that $supp(\rho) \subseteq supp(B)$, then for any $\varepsilon >0$ we can choose a $\delta >0$ such that:
\[
A' = A + \delta \rho \implies A' \hypo_{k'} B \text{ and } |k - k'| < \varepsilon.
\]
\end{thm}

\subsection{Scaling}
\label{sec:norm}
When comparing positive operators, in order to standardize the magnitudes resulting from calculations,
it is natural to consider normalizing their trace so that we work with density operators. Unfortunately,
this is a poor choice when working with the L\"owner order as distinct pairs of density operators are
never ordered with respect to each other. Instead we consider bounding our operators as having maximum
eigenvalue 1, as suggested in~\citep{DHondt2006}. With this ordering, the projection operators regain
their usual ordering and we recover quantum logic as a suborder of our setting.

Our framework is flexible enough to support other normalization strategies. The optimal choice for linguistic
applications is left to future empirical work. More interesting ideas are also possible.
For example we can embed the Bayesian order~\cite{coecke2011a} within our setting via a suitable transformation on positive operators.
This is described in more detail in appendix section~\ref{sec:Bayesian}.
Further theoretical investigations of this type are left to future work.

\subsection{Examples}
In this section we give three simple examples and illustrate the order for 2-dimensional matrices in the Bloch sphere.

\begin{exa}\em
\label{exa:1}
Consider the density matrix:
\begin{equation*}
  \semantics{\lang{pet}} = 1/2\ket{\lang{dog}}\bra{\lang{dog}} + 1/2\ket{\lang{cat}}\bra{\lang{cat}}
\end{equation*}
The entailment strength~$k$ such that~$k\ket{\lang{dog}}\bra{\lang{dog}} \leq \semantics{\lang{pet}}$ is~$\frac{1}{2}$.
\end{exa}
Further, if two mixed states~$\rho$, $\sigma$ can both be expressed as convex combinations of the same two pure states,
the extent to which one state entails the other can also be derived.
\begin{exa}\em
\label{exa:2}
For states:
\begin{align*}
  \rho &= r\ket{\psi}\bra{\psi} + (1 - r)\ket{\phi}\bra{\phi}\\
  \sigma &= s\ket{\psi}\bra{\psi} + (1 - s)\ket{\phi}\bra{\phi}
\end{align*}
the entailment strength~$k$ such that~$k\sigma \leq \rho$ is given by:
\[
 k =
  \begin{cases}
   \frac{r}{s} & \text{if } r < s\\
   \frac{(1 - r)}{(1 - s)} & \text{otherwise}
  \end{cases}
\]
\end{exa}

\begin{exa}\em
Suppose that $\semantics{B} = k_j \semantics{A} + \sum_{i \not = j} k_i \: \semantics{X_i} $. Then:
\[ 
\semantics{A} \hypo_{k} \semantics{B}
\]
for any $k \leq k_j$.
\end{exa}

From the above example we notice that the value~$k_1$ definitely gives us $k_1$-hyponymy between $A$ and $B$,
but it is actually possible that there exists a value, say~$l$, such that~$l > k_1$ and for which we have $l$-hyponymy between~$A$ and~$B$.
Indeed, this happens whenever we have an~$l$ for which:
\begin{equation*}
  (k_1 - l)\bra{x} \semantics{A} \ket{x} \geq - \sum_{i\not=1} p_i \bra{x} \semantics{X_i} \ket{x}
\end{equation*}
Thus, $k_1$ may not be the maximum value for hyponymy between~$A$ and~$B$.
In general, however, we are interested in making the strongest assertion we can and therefore we are interested in the maximum value of~$k$,
which we call the \emph{entailment strength}.
For matrices on~$\mathbb{R}^2$, we can represent these entailment strengths visually using the Bloch sphere restricted to~$\mathbb{R}^2$ - the `Bloch disc'.

\subsubsection{Representing the order in the `Bloch disc'}
The Bloch sphere, \cite{Bloch1946}, is a geometrical representation of quantum states. Very briefly, points on the sphere correspond to pure states, and states within the sphere to impure states. Since we consider matrices only over $\mathbb{R}^2$, we disregard the complex phase which allows us to represent the pure states on a circle. A pure state  $\cos(\theta/2)\ket{0} + \sin(\theta/2)\ket{1}$ is represented by the vector $(\sin(\theta), \cos(\theta))$ on the circle. 

We can calculate the entailment factor~$k$ between any two points on the disc.
For example, in figure  \ref{fig:entailment_strengths} we show contour maps of the entailment strengths for the state with Bloch vector $(\frac{3}{4} \sin(\pi/5), \frac{3}{4} \cos(\pi/5))$, using the maximum eigenvalue normalization.

\begin{figure}[htbp]
  \centering
  \begin{tikzpicture}
    \draw (0, 6) node {$\ket{0}$};
    \draw (-3.25, 3) node {$\frac{\ket{0} - \ket{1}}{\sqrt{2}}$};
    \draw (0, 3) node {\includegraphics[width=0.65\linewidth]{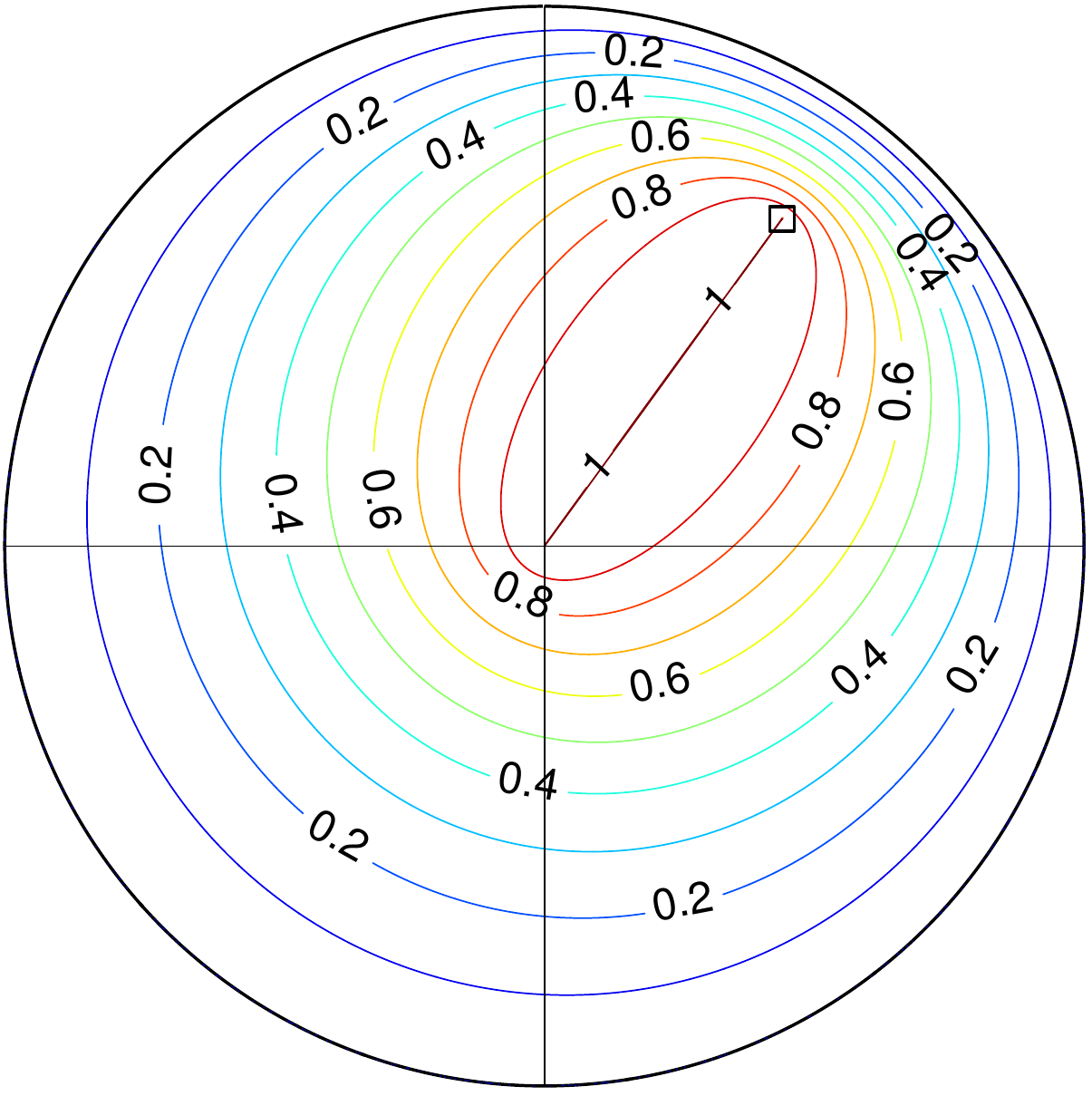}};
    \draw (3.25, 3) node {$\frac{\ket{0} + \ket{1}}{\sqrt{2}}$};
    \draw (0, 0) node {$\ket{1}$};
  \end{tikzpicture}
  \caption{Entailment strengths in the Bloch disc for the state with Bloch vector $(\frac{3}{4} \sin(\pi/5), \frac{3}{4} \cos(\pi/5))$.}
  \label{fig:entailment_strengths}
\end{figure}
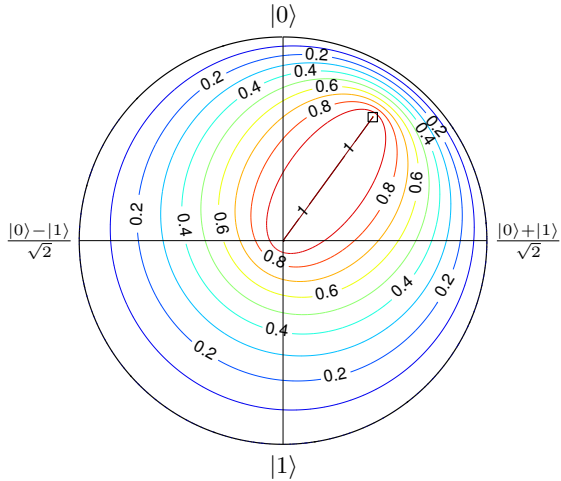

\section{Main Results on Compositionality} 

We will now consider what happens when we have two sentences such that one of them contains one or more hyponyms of one or more words from the other. We will show that in this case the hyponymy is `lifted' to the sentence level, and that the $k$-values are preserved in a very intuitive fashion. After considering a couple of specific sentence constructions, we will generalise this result to account for a broad category of sentence patterns that work in the compositional distributional model. 
\subsection{$k$-hyponymy in positive transitive sentences} 
A positive transitive sentence has the diagrammatic representation in~$\CPM{\FHilb}$ given in figure~\ref{fig:mixed_tscpmc}.
The meaning of the sentence~\lang{subj verb obj} is given by:
     \[
     (\ep_N \ot 1_S \ot \ep_N)\left( \semantics{\lang{subj}} \ot \semantics{\lang{verb}} \ot \semantics{\lang{obj}} \right) \: ,
     \]
where the~$\ep_N$ and~$1_S$ morphisms are those from~$\CPM{\FHilb}$. We will represent the subject and object by: 
     \[
          \semantics{\lang{subj}} = \sum_{ik} a_{ik} \ket{n_i}\bra{n_k}  \hspace{0.5cm} \text{and} \hspace{0.5cm} \semantics{obj} = \sum_{jl} b_{jl} \ket{n_j}\bra{n_l} .
     \]
Finally, let the verb be given by: 
     \[
          \semantics{\lang{verb}} = \sum_{pqrtuv} C_{pqrtuv}  \ket{n_p} \bra{n_t} \ot \ket{s_q} \bra{s_u} \ot \ket{n_r} \bra{n_v} 
     \]

     \dm{I rewrote the below slightly, stripping out some redundancy. Check Martha that it's still ok and discuss improvements}
\begin{thm}
\label{thm:trans}
Let~$n_1, n_2, n_3, n_4$ be nouns with  corresponding density matrix representations~$\semantics{n_1}$, $\semantics{n_2}$, $\semantics{n_3}$ and~$\semantics{n_4}$,
such that~$n_1$ is a $k$-hyponym of~$n_2$ and~$n_3$ is a $l$-hyponym of~$n_4$. Then:
          \[
               \varphi \left( n_1  \: \lang{verb} \:  n_3  \right) \hypo_{kl} \varphi \left(  n_2 \: \lang{verb} \:  n_4 \right),
          \]
where~$\varphi = \ep_N \ot 1_S \ot \ep_N$ is the sentence meaning map for positive transitive sentences. 
\end{thm}
\dm{I think this theorem could be clearer, and again could do with slightly more fanfare! Does it make any assumptions about the nature of the verb?}

\subsection{General Sentence $k$-hyponymy}
We can show that the application of $k$-hyponymy to various phrase types holds in the same way. In this section we provide a general proof for varying phrase types. We adopt the following conventions: 
\begin{itemize}
\item A \emph{positive phrase} is assumed to be a phrase in which individual words are upwardly monotone in the sense described by \cite{maccartney2007}. This means that, for example, the phrase does not contain any negations, including words like \emph{not}. 
\item The \emph{length} of a phrase is the number of words in it, not counting definite and indefinite articles. 
\end{itemize}

\begin{thm}[Generalised Sentence $k$-Hyponymy]
\label{thm:general}
Let~$\Phi$ and~$\Psi$ be two positive phrases of the same length and grammatical structure, expressed in the same noun spaces $N$ and sentence spaces $S$. Denote the nouns and verbs of~$\Phi$, in the order in which they appear, by~$A_1, \ldots , A_n$. Similarly, denote these in~$\Psi$ by~$B_1 \ldots B_n$. Let their corresponding density matrices be
denoted by~$\semantics{A_1}, \ldots , \semantics{A_n}$ and~$\semantics{B_1}, \ldots , \semantics{B_n}$ respectively.
Suppose that~$\semantics{A_i} \hypo_{k_i} \semantics{B_i}$ for~$i \in \{ 1 , \ldots , n \} $ and some~$k_i \in (0,1]$.
Finally, let~$\varphi$ be the sentence meaning map for both~$\Phi$ and~$\Psi$,
such that~$\varphi(\Phi)$ is the meaning of~$\Phi$ and~$\varphi(\Psi)$ is the meaning of~$\Psi$.
Then: 
\[
    \varphi(\Phi) \hypo_{k_1 \cdots k_n} \varphi(\Psi).
\]
so $k_1 \cdots k_n$ provides a lower bound on the extent to which $\varphi(\Phi)$ entails $\varphi(\Psi)$
\end{thm} 
Intuitively, this means that if (some of) the functional words of a sentence~$\Phi$ are $k$-hyponyms
of (some of) the functional words of sentence~$\Psi$, then this hyponymy is translated into sentence hyponymy. 
Upward-monotonicity is important here, and in particular implicit quantifiers. It might be objected that~\lang{dogs bark} should not imply~\lang{pets bark}. 
If the implicit quantification is universal, then this is true, however the universal quantifier is downward monotone, and therefore does not conform to the convention concerning positive phrases. 
If the implicit quantification is existential, then \lang{some dogs bark} does entail \lang{some pets bark}, and the problem is averted. 
Discussion of the behaviour of quantifiers and other word types is given in~\cite{maccartney2007}.

The quantity~$k_1 \cdots k_n$ is not necessarily maximal, and indeed usually is not. As we only have a lower bound, zero entailment strength between a pair of components
does not imply zero entailment strength between entire sentences. Results for phrases involving relative clauses may be found in appendix \ref{app:frobenius}.
\begin{cor}
\label{cor:strict}
Consider two sentences:
\begin{equation*}
  \Phi = \bigotimes_i \semantics{A_i} \qquad \Psi = \bigotimes_i \semantics{B_i}
\end{equation*}
such that for each~$i \in \{1,..., n\}$ we have~$\semantics{A_i} \loewner \semantics{B_i}$,
i.e. there is strict entailment in each component. Then there is strict entailment between the sentences~$\varphi(\Phi)$ and~$\varphi(\Psi)$.
\end{cor}

We consider a concrete example.
\paragraph{Compositionality of $k$-hyponymy in a transitive sentence.}
More examples may be found in appendix \ref{app:examples}. Suppose we have a noun space $N$ with basis $\{\ket{e_i}\}_i$, and sentence space $S$ with basis $\{\ket{x_j}\}_j$
We consider the verbs $\lang{nibble}$, $\lang{scoff}$ and the nouns $\lang{cake}$, $\lang{chocolate}$, with semantics:
\ml{Perhaps I don't need to write all of these out.}
\begin{align*}
     \semantics{\lang{nibble}} & = \sum_{pqrtuv} a_{pqr}a_{tuv}\ket{e_p}\bra{e_t} \ot \ket{x_q}\bra{x_u} \ot \ket{e_r}\bra{e_v}\\
	\semantics{\lang{scoff}} & = \sum_{pqrtuv} b_{pqr}b_{tuv}\ket{e_p}\bra{e_t} \ot \ket{x_q}\bra{x_u} \ot \ket{e_r}\bra{e_v}\\
	\semantics{\lang{cake}} & = \sum_i c_ic_j\ket{e_i}\bra{e_j}\\
	\semantics{\lang{chocolate}} & = \sum_i d_id_j\ket{e_i}\bra{e_j}
\end{align*} 
which make these nouns and verbs pure states. The more general $\lang{eat}$ and $\lang{sweets}$ are given by:
\begin{align*}
     \semantics{\lang{eat}} & = \frac{1}{2}(\semantics{\lang{nibble}} + \semantics{\lang{scoff}})\\
     \semantics{\lang{sweets}} & = \frac{1}{2}(\semantics{\lang{cake}} + \semantics{\lang{chocolate}})
\end{align*} 
Then 
\begin{align*}
\semantics{\lang{scoff}} &\hypo_{1/2} \semantics{\lang{eat}}\\
\semantics{\lang{cake}} &\hypo_{1/2} \semantics{\lang{sweets}}
\end{align*}
We consider the sentences:
\begin{align*}
     s_1 &= \lang{John scoffs cake} \\
     s_2 &= \lang{John eats sweets}
\end{align*} 
The semantics of these sentences are:
\begin{align*}
\semantics{s_1} &= \varphi(\semantics{\lang{Mary}} \ot \semantics{\lang{scoffs}} \ot \semantics{\lang{cake}})\\
\semantics{s_2} &= \varphi(\semantics{\lang{Mary}} \ot \semantics{\lang{eats}} \ot \semantics{\lang{sweets}})
\end{align*}
and as per theorem~\ref{thm:general}, we will show that~$\semantics{s_1} \hypo_{kl} \semantics{s_2}$ where~$kl = \frac{1}{2} \times \frac{1}{2} = \frac{1}{4}$
Expanding~$\semantics{s_2}$ we obtain:
\begin{align*}
\semantics{s_2} &= \varphi(\semantics{\lang{Mary}} \ot \frac{1}{2}(\semantics{nibbles} + \semantics{\lang{scoffs}})\\
&\qquad \ot \frac{1}{2}(\semantics{\lang{cake}} + \semantics{\lang{choc}}))\\
&=\frac{1}{4}(\varphi(\semantics{\lang{Mary}} \ot \semantics{\lang{scoffs}}\ot \semantics{\lang{cake}})\\
&\qquad + \varphi(\semantics{\lang{Mary}} \ot \semantics{\lang{scoffs}}\ot \semantics{\lang{choc}})\\
&\qquad + \varphi(\semantics{\lang{Mary}} \ot \semantics{\lang{nibbles}}\ot \semantics{\lang{cake}})\\
&\qquad + \varphi(\semantics{\lang{Mary}} \ot \semantics{\lang{nibbles}}\ot \semantics{\lang{choc}}))\\
&=\frac{1}{4}\semantics{s_1} + \frac{1}{4}(\varphi(\semantics{\lang{Mary}} \ot \semantics{\lang{scoffs}}\ot \semantics{\lang{choc}})\\
&\qquad + \varphi(\semantics{\lang{Mary}} \ot \semantics{\lang{nibbles}}\ot \semantics{\lang{cake}})\\
&\qquad + \varphi(\semantics{\lang{Mary}} \ot \semantics{\lang{nibbles}}\ot \semantics{\lang{choc}}))
\end{align*}
Therefore:
\begin{align*}
\semantics{s_2} - \frac{1}{4}\semantics{s_1} & = \varphi(\semantics{\lang{Mary}} \ot \semantics{\lang{choc}}\ot \semantics{\lang{choc}})\\
&\qquad + \varphi(\semantics{\lang{Mary}} \ot \semantics{\lang{nibbles}}\ot \semantics{\lang{cake}})\\
&\qquad + \varphi(\semantics{\lang{Mary}} \ot \semantics{\lang{nibbles}}\ot \semantics{\lang{choc}}))
\end{align*}
We can see that~$\semantics{s_2} - \frac{1}{4}\semantics{s_1}$ is positive by
positivity of the individual elements and the fact that positivity is preserved under addition and tensor product.  Therefore: 
\[
\semantics{s_1} \hypo_{kl} \semantics{s_2}
\]
as required.

\section{Conclusion}
Integrating a logical framework with compositional distributional semantics is an important step in improving this model of language.
By moving to the setting of density matrices, we have described a graded measure of entailment that may be used to describe the extent of
entailment between two words represented within this enriched framework. This approach extends uniformly to provide entailment strengths
between phrases of any type. We have also shown how a lower bound on entailment strength of phrases of the same structure can be
calculated from their components.

We can extend this work in several directions. Firstly, we can examine how narrowing down a concept using an adjective might operate.
For example, we should have that~\lang{red car} entails~\lang{car}. Other adjectives should not operate in this way, such as~\lang{former} in~\lang{former president}.
 
Another line of inquiry is to examine transitivity behaves. In some cases entailment can strengthen.
We had that \lang{dog} entails \lang{pet} to a certain extent,
and that \lang{pet} entails \lang{mammal} to a certain extent, but that \lang{dog} completely entails \lang{mammal}.

Our framework supports different methods of scaling the positive operators representing propositions. Empirical work will
be required to establish the most appropriate method in linguistic applications.
 
Sentences with non-identical structure must also be taken into account.
One approach to this might be to look at the first stage  in the sentence reductions at which the elementwise comparison can be made.
For example, in the two sentences \lang{John runs very slowly}, and \lang{Hungry boys sprint quickly},
we can compare the noun phrases \lang{John}, and \lang{Hungry boys}, the verbs \lang{runs} and \lang{sprints},
and the adverb phrases \lang{very slowly} and \lang{quickly}.
Further, the inclusion of negative words like \lang{not}, or negative prefixes, should be modelled.

\acks
Bob Coecke, Martha Lewis, and Daniel Marsden gratefully acknowledge funding from AFOSR grant Algorithmic and Logical Aspects when Composing Meanings.


\bibliographystyle{abbrvnat}
\bibliography{hyponymy}
%
%
%

\clearpage

\appendix
\section{Proofs}
\label{app:proofs}

\begin{proof}[Proof of Theorem \ref{thm:kmax}]
  We wish to find the maximum~$k$ for which
  \begin{equation*}
    \forall \ket{x} \in \mathbb{R}^n .\bra{x}(B - pA)\ket{x} \geq 0
  \end{equation*}
  Since~$supp(A) \subseteq supp(B)$, such a~$k$ exists.
  We assume that for~$k = 1$, there is at least one~$\ket{x}$ such that~$\bra{x}(B - kA)\ket{x} \leq 0$, since otherwise we're done.
  For all~$\ket{x} \in \mathbb{R}^n$, $\bra{x}(B - kA)\ket{x}$ increases continuously as~$k$ decreases.
  We therefore decrease~$k$ until~$\bra{x}(B - kA)\ket{x} \geq 0$, and there will be at least one~$\ket{x_0}$ at which~$\bra{x_0}(B - kA)\ket{x_0} = 0$.
  These points are minima so that the vector of partial derivatives~$\nabla \bra{x_0}(B - k_0A)\ket{x_0}$ will be zero everywhere.
\[
\nabla \bra{x_0}(B - k_0A)\ket{x_0} = 2(B - k_0A)\ket{x_0} = \overrightarrow{0}
\]
(requires $B$, $A$ self-adjoint).

Therefore~$B\ket{x_0} = k_0A\ket{x_0}$, and so~$1/k_0B^+B\ket{x_0} = B^+A\ket{x_0}$. Since~$B^+ B$ is a projector onto the support of~$B$ and~$supp(A) \subseteq supp(B)$, we have: 
\[
1/k_0\ket{v_0} = B^+A\ket{v_0}
\]
where~$\ket{v_0} = B^+B\ket{x_0}$, i.e., $1/k_0$ is an eigenvalue of~$B^{+}A$. 

Now, $B^+A$ has only non-negative eigenvalues, and in fact any pair of eigenvalue~$1/k$ and eigenvector~$\ket{v}$ will satisfy the condition~$B\ket{v} = kA\ket{v}$.
We now claim that to satisfy~$\forall \ket{x} \in \mathbb{R}^n .\bra{x}(B - kA)\ket{x} \geq 0$,
we must choose~$k_0$ equal to the reciprocal of the maximum eigenvalue~$\lambda_0$ of $B^+A$.
For a contradiction, take~$\lambda_1 < \lambda_0$, so~$1/\lambda_1 = k_1 > k_0 = 1/\lambda_0$.
Then we require that~$\forall \ket{x} \in \mathbb{R}^n .\bra{x}(B - k_1A)\ket{x} \geq 0$, and in particular for~$\ket{v_0}$. However:
\begin{align*}
  \bra{v_0}(B - k_1A)\ket{v_0} \geq 0 &\iff \bra{v_0}B\ket{v_0} \geq k_1\bra{v_0}A\ket{v_0}\\
  &\iff k_0\bra{v_0}A\ket{v_0} \geq k_1\bra{v_0}A\ket{v_0}\\
  &\text{contradiction, since $k_0 < k_1$}
\end{align*}
We therefore choose~$k_0$ equal to~$1/\lambda_0$ where~$\lambda_0$ is the maximum eigenvalue of~$B^+A$,
and~$\bra{x}(B - k_0A)\ket{x} \geq 0$ is satisfied for all~$\ket{x} \in \mathbb{R}^n$. 
\end{proof}

\begin{proof}[Proof of Theorem \ref{thm:continuity}]
We wish to show that we can choose $\delta$ such that $|k - k'| < \varepsilon$. We use the notation $\lambda_\lang{max}(A)$ for the maximum eigenvalue of $A$, and $A^+$ for the Moore-Penrose pseudo-inverse of $A$.
$A' = A + \delta\rho$ satisfies the condition of theorem \ref{thm:kmax}, that $supp(A') \subseteq supp(B)$, since suppose $\ket{x} \not \in supp(B)$. $supp(A) \subseteq supp(B)$, so $\ket{x} \not \in supp(A)$ and $A\ket{x} = 0$. Similarly, $\rho\ket{x} = 0$. Therefore $(A + \rho)\ket{x} = A'\ket{x} = 0$, so $\ket{x} \not \in supp(A')$.

By theorem  \ref{thm:kmax} we have:
\begin{align*}
k &= \frac{1}{\lambda_\lang{max}(B^+A)}\\
k'&= \frac{1}{\lambda_\lang{max}(B^+A')}\\
\end{align*}
\begin{equation}
\label{eq:diffk}
k - k' = \frac{\lambda_\lang{max}(B^+A') - \lambda_\lang{max}(B^+A)}{\lambda_\lang{max}(B^+A')\lambda_\lang{max}(B^+A)}
\end{equation}
The denominator of \ref{eq:diffk} we may treat as a constant. We expand the numerator and apply Weyl's inequalities \cite{weyl1912}. These inequalities apply only to Hermitian matrices, whereas we need to apply these to products of Hermitian matrices. Note that since $B^+$, $A$, and $\rho$ are all real-valued positive semidefinite, the products $B^+A$ and $B^+\rho$ have the same eigenvalues as the Hermitian matrices $A^\frac{1}{2}B^+A^\frac{1}{2}$ and $\rho^\frac{1}{2}B^+\rho^\frac{1}{2}$ which are Hermitian. Now:
\begin{align*}
\lambda_\lang{max}(B^+A') - \lambda_\lang{max}(B^+A) &= \lambda_\lang{max}(B^+A + \delta B^+\rho) - \lambda_\lang{max}(B^+A)\\
& \leq \lambda_\lang{max}(B^+A) + \delta \lambda_\lang{max}(B^+\rho) - \lambda_\lang{max}(B^+A)\\ 
& = \delta \lambda_\lang{max}(B^+\rho)\\
& \leq \delta \lambda_\lang{max}(B^+)\lambda_\lang{max}(\rho) \leq \delta \lambda_\lang{max}(B^+)
\end{align*}
Therefore:
\begin{equation}
k - k' \leq \delta \frac{\lambda_\lang{max}(B^+)}{\lambda_\lang{max}(B^+A')\lambda_\lang{max}(B^+A)}
\end{equation}
so that given $\varepsilon$, $A$, $B$, we can always choose a $\delta$ to make $k - k' \leq \varepsilon$
\end{proof}

\begin{proof}[Proof of Theorem \ref{thm:trans}]
  Let the density matrix corresponding to the verb be given by~$\semantics{Z}$ and the linear map~$(\ep_N \ot 1_S \ot \ep_N)$ be given by~$\varphi$.
  Then we can write the meanings of the two sentences as: 
\[
 \varphi \left( \semantics{n_1} \otimes \semantics{verb} \otimes{n_3} \right)= \varphi \left( \semantics{n_1} \ot \semantics{Z} \ot \semantics{n_3} \right)
\]
\[
\varphi \left( \semantics{n_2} \otimes \semantics{verb} \otimes \semantics{n_4} \right) = \varphi \left( \semantics{n_2} \ot \semantics{Z} \ot \semantics{n_4} \right)
\]
Substituting~$\semantics{n_2} = k\semantics{n_1} + D$ and~$\semantics{n_4} = l \: \semantics{n_3} + D'$ in the expression for the meaning of~``$n_2$ verb $n_4$'' gives:
\begin{align*}
&\varphi(\semantics{n_2} \ot \semantics{Z} \ot \semantics{n_4}) = \varphi (( k\, \semantics{n_1} + D ) \ot \semantics{Z} \ot ( l \, \semantics{n_3} + D')) \\                                                                                     
& \qquad = kl\varphi(\semantics{n_1} \ot \semantics{Z} \ot \semantics{n_3}) + \varphi(( k\: \semantics{n_1} \ot \semantics{Z} \ot D' ) \\
& \qquad \qquad + (D \ot \semantics{Z} \ot l\, \semantics{n_3})  + (D \ot \semantics{Z} \ot D' ))\\
\end{align*}
Therefore, $\varphi(\semantics{n_2} \ot \semantics{Z} \ot \semantics{n_4} ) - kl\,\varphi(\semantics{n_1} \ot \semantics{Z} \ot \semantics{n_3} )$
is equal to~$\varphi(( k\: \semantics{n_1} \ot \semantics{Z} \ot D' ) + (D \ot \semantics{Z} \ot l\, \semantics{n_3})  + (D \ot \semantics{Z} \ot D' ))$
which is positive by positivity of~$\semantics{n_1}$, $\semantics{Z}$, $\semantics{n_3}$, $D$, $D'$, and the scalars~$k$ and~$l$.
Therefore:
\[
\varphi(\semantics{n_1} \ot \semantics{Z} \ot \semantics{n_3} ) \hypo_{kl} \varphi(\semantics{n_2} \ot \semantics{Z} \ot \semantics{n_4} ) 
\]
as needed.
\end{proof}

\begin{proof}[Proof of Theorem \ref{thm:general}]
  First of all, we have~$\semantics{A_i} \hypo_{k_i} \semantics{B_i}$ for~$i \in \{1, ..., n\}$.
  This means that for each $i$, we have positive matrices~$\rho_{i}$ and non-negative reals~$k_i$ such that~$\semantics{B_i} = k_i \semantics{A_i} + \rho_i$.
  Now consider the meanings of the two sentences. We have: 
  \begin{align*}
    \varphi(\Phi) &= \phi(\semantics{A_1} \ot \ldots \ot \semantics{A_n}), \\
    \varphi(\Psi) & = \varphi(\semantics{B_1} \ot \ldots \ot \semantics{B_n}) \\
    & = \varphi \left( (k_1 \semantics{A_1} + \rho_1) \ot \ldots \ot (k_n\, \semantics{A_n} + \rho_n  \right)\\
    &= (k_1 \cdots k_n)\varphi(\semantics{A_1} \ot ... \ot \semantics{A_n}) + \varphi(P)
  \end{align*}
  where~$P$ consists of a sum of tensor products of positive matrices, namely:
  \[
  P = \sum_{S\subset\{1, ..., n\}} \bigotimes_{i = 1}^n \sigma_i
  \]
  where:
  \begin{align}
    \sigma_i = \begin{cases}
      k_i \semantics{A_i} &\quad  \text{ if $i \in S$}\\
      \rho_i & \quad \text{ if $i \not \in S$}\\
    \end{cases}
  \end{align}
  Then we have:
  \[
  \varphi(\Psi) - (k_1 ... k_n) \varphi(\Phi) = \varphi(P) \geq 0
  \]
  since~$P$ is a sum of tensor products of positive matrices, and~$\varphi$ is a completely positive map. Therefore:
  \[
  \varphi(\Phi) \hypo_{k_1 \cdots k_n} \varphi(\Psi)
  \]
  as required.
\end{proof}

\begin{proof}[Proof of Corollary \ref{cor:strict}]
  Since~$k_i = 1$ for each~$i = \{1, ..., n\}$, 
  \begin{align*}
    \varphi(\Phi) \hypo_{k_1 \cdots k_n} \varphi(\Psi) &\implies \varphi(\Phi) \hypo_1 \varphi(\Psi)\\
    & \implies \varphi(\Phi) \leq \varphi(\Psi)
  \end{align*}
\end{proof}
We refer to the following simple factor, verified here:
\begin{lem}
  Let~$\sigma, \tau$ be density operators. Then:
  \begin{equation*}
    \sigma \loewner \tau \quad\Rightarrow\quad \sigma = \tau
  \end{equation*}
\end{lem}
\begin{proof}
  If $\sigma \loewner \tau$ then $\tau - \sigma$ is positive, and clearly:
  \begin{equation*}
    \sigma + (\tau - \sigma) = \tau
  \end{equation*}
  therefore, applying the trace and noting density operators all have trace 1:
  \begin{equation*}
    \trace(\tau - \sigma) = 0
  \end{equation*}
  and as $\tau - \sigma$ is positive, it must be the zero operator.
\end{proof}

\subsection{Incorporating the Bayesian Order}
\label{sec:Bayesian}
We can work with the Bayesian order on density operators~\citep{coecke2011a}. In order
to do this, we apply the following operations to transform our density operators:
\begin{enumerate}
\item Diagonalize the operator, choosing a permutation of the basis vectors such that the diagonal elements are in descending order.
\item Let~$d_i$ denotes the $i^{th}$ diagonal element. We define the diagonal of a new diagonal matrix inductively as follows:
  \begin{equation*}
    d'_0 = d_0 \qquad\qquad d'_{i+1} = d'_{i} * d_{i + 1}
  \end{equation*}
\item Transform the new operator back to the original basis
\end{enumerate}

\section{Examples}
\label{app:examples}
\subsection{Examples of $k$-hyponymy in positive transitive sentences}
\label{sec:ktrans}
In this section we give three toy examples of the use of $k$-hyponymy in positive transitive sentences. Each example uses a different sentence space.

\subsubsection{Truth-theoretic sentence spaces} 
The following example illustrates what happens to positive transitive sentence hyponymy if we take a truth-theoretic approach to sentence meaning. Suppose that our sentence space $S$ is 1-dimensional, with its single non-trivial vector being $\ket{1}$ . We will take $\ket{1}$ to stand for \emph{True} and the origin $0$ for \emph{False}.
The sentences we will consider are: 
\begin{align*}
  s_1 &= \lang{Annie enjoys holidays}\\
  s_2 &= \lang{Students enjoy holidays} 
\end{align*} 
Let the vector space for the subjects of the sentences be~$\mathbb{R}^3$ with chosen basis~$\{ \ket{e_1}, \ket{e_2}, \ket{e_3}  \}$. Let:
\[
\semantics{\lang{Annie}} = \ket{e_1}\bra{e_1}, \, \semantics{\lang{Betty}} = \ket{e_2}\bra{e_2}, \, \semantics{\lang{Chris}} = \ket{e_3}\bra{e_3}
\]
Let the object vector space be~$\mathbb{R}^n$ for some arbitrary~$n \in \mathbb{N}$,
where we take~$\{\ket{v_i}\}_i$ to be the standard basis for~$\mathbb{R}^n$,
where~$\ket{v_i}$ has 1 in the~$i$th position and 0 elsewhere. Let $\ket{\lang{holidays}} = \ket{v_1}$. We
will treat the word \lang{students} as being a hypernym of the individual students in our universe. 
\[
\semantics{\lang{students}} = \frac{1}{3} \, \semantics{\lang{Annie}} + \frac{1}{3} \, \semantics{\lang{Betty}} + \frac{1}{3} \, \semantics{\lang{Chris}}
\]
We have the choices for normalization that we outlined in section~\ref{sec:norm}.
Since we are viewing the sentence space as truth-theoretic, we keep the normalization to trace 1.

Finally, let the verb be given by:
\[
\semantics{\lang{enjoy}} = \sum_{\substack{(p,q)\in R \\ (r,s) \in R}} \ket{e_p}\bra{e_r} \ot \ket{v_q}\bra{v_s}.
\]
where~$R = \{(i, j) | \ket{e_i} enjoys \ket{v_j}\}$

Suppose that Annie and Betty are known to enjoy holidays, while Chris does not. Clearly, we have that:
\[
\semantics{\lang{Annie}} \hypo_{k_{max}} \semantics{\lang{students}}
\]for $k_{max} = \frac{1}{3}$ since:
\[
\semantics{\lang{students}} - \frac{1}{3} \semantics{\lang{Annie}} = \frac{1}{3}(\semantics{\lang{Betty}} + \semantics{\lang{Chris}}) \geq 0
\] 
and any higher value than~$\frac{1}{3}$ will no longer be positive.

We will see that the $k$-hyponymy for~$k = \frac{1}{3}$ does translate into~$k$-hyponymy of sentence~$s_1$ to sentence~$s_2$.
First of all, consider the meanings of the two sentences:
\dm{Are these two correct? I see $v_i$ terms appearing in the expression for $s_1$, but not for $s_2$}
\ml{They were correct, but I have added a few extra lines to clarify.} 
\begin{align*}
  \semantics{s_1} & = (\ep_N \ot 1_S \ot \ep_N ) (\semantics{\lang{Annie}} \ot \semantics{\lang{enjoys}} \ot \semantics{\lang{holidays}}) \\
  &= (\ep_N \ot 1_S \ot \ep_N )(\ket{e_1}\bra{e_1}\\
  & \qquad \ot \left(\sum_{\substack{(p,q)\in R \\ (r,s) \in R}} \ket{e_p}\bra{e_r} \ot \ket{v_q}\bra{v_s} \right) \ot \ket{v_1}\bra{v_1})\\
&= \sum_{\substack{(p,q)\in R \\ (r,s) \in R}} \braket{ e_1  | e_p} \braket{ e_1 | e_r} \braket{ v_q | v_1}\braket{ v_s | v_1 } \\
&= \sum_{\substack{(p, q)\in R \\ (r,s) \in R}} \delta_{1p}\, \delta_{1r} \, \delta_{q1} \, \delta_{s1} = \sum_{\substack{(1,1)\in R \\ (1,1) \in R}} 1 \: \: \: = 1 
\end{align*}

\begin{align*}
\semantics{s_2} & = (\ep_N \ot 1_S \ot \ep_N) \left( \semantics{students} \ot \semantics{enjoy} \ot \semantics{holidays} \right) \\
&= (\ep_N \ot 1_S \ot \ep_N )(\frac{1}{3}(\ket{e_1}\bra{e_1} + \ket{e_2}\bra{e_2} + \ket{e_3}\bra{e_3}
\\& \qquad \ot \left(\sum_{\substack{(p,q)\in R \\ (r,s) \in R}} \ket{e_p}\bra{e_r} \ot \ket{v_q}\bra{v_s} \right) \ot \ket{v_1}\bra{v_1})\\
& = \frac{1}{3}   \sum_{\substack{(p,q)\in R \\ (r,s) \in R}} ( \braket{ e_1 | e_p}\braket{ e_1 | e_r} + \braket{ e_2 | e_p}\braket{ e_2 | e_r} \\
& \qquad +  \braket {e_3 | e_p}\braket{ e_3 | e_r })\braket{ v_q | v_1}\braket{ v_s | v_1 }\\
& = \frac{1}{3}   \sum_{\substack{(p,1)\in R \\ (r,1) \in R}} ( \braket{ e_1 | e_p}\braket{ e_1 | e_r} + \braket{ e_2 | e_p}\braket{ e_2 | e_r} \\
&\qquad +  \braket {e_3 | e_p}\braket{ e_3 | e_r })\\
&  = \frac{1}{3} \times 2 = \frac{2}{3} 
\end{align*}
Clearly, we have that~$\semantics{s_1} \hypo_k \semantics{s_2} $ for~$k = \frac{1}{3}$, as $\frac{2}{3} - \frac{1}{3}\times 1 \geq 0$,
but this is not the maximum value of~$k$ for which this $k$-hyponymy holds.
The maximum value for which this works is~$k= \frac{2}{3}$.
  
\subsubsection{Simple case of object hyponymy}
We now give a simple case with a non-truth-theoretic sentence space.
We show that the $k$-hyponymy of the objects of two sentences translates into $k$-hyponymy between the sentences,
and that in this case the maximality of the value of~$k$ is also preserved. 

Let~$m \in \mathbb{N}$, $m > 2$ be such that $\{ \ket{n_i} \}_{i=1}^{m}$ is a collection of standard basis vectors for~$\mathbb{R}^m$. We will use the nouns: 
\[
\semantics{\lang{Gretel}} = \ket{n_1}\bra{n_1}, \hspace{0.3cm} \semantics{\lang{gingerbread}} = \ket{n_2} \bra{n_2}
\]
\[
\semantics{\lang{cake}} = \ket{n_3} \bra{n_3}, \hspace{0.3cm} \semantics{\lang{pancakes}} = \ket{n_4} \bra{n_4}
\]
Let the density matrix corresponding to the hypernym \lang{sweets} be given by: 
\[
\semantics{\lang{sweets}} = \frac{1}{10} \ket{n_2}\bra{n_2} + \sum_{i=3}^{m} p_i \, \ket{n_i}\bra{n_i}.
\]
Our object and subject vector space will be~$\mathbb{R}^m$ and for the sentence space we take~$S = \mathbb{R}^m \ot \mathbb{R}^m$.
Using this sentence space simplifies the calculations needed, as shown in~\cite{Grefenstette2011}.
For the rest of this example, we will adopt the following of notation for the purpose of brevity: 
\[
\ket{x_{jk}} = \ket{n_j}\ket{n_k}, \hspace{0.3cm} \bra{x_{jk}} = \bra{n_j} \bra{n_k}, 
\]
\[
\ket{x_{ij}}\bra{x_{kl}} = \ket{n_i}\bra{n_k} \ot \ket{n_j}\bra{n_l}.
\]
Then the density matrix representation of our verb becomes:
\[
\semantics{\lang{likes}} = \sum_{jklp} C_{jklp}\ket{n_j}\bra{n_l} \ot \ket{x_{jk}}\bra{x_{lp}} \ot \ket{n_k}\bra{n_p}
\]
 We will consider the following two sentences:
\begin{align*}
     s_1 &= \lang{Gretel likes sweets} \\
     s_2 &= \lang{Gretel likes gingerbread} 
\end{align*} 
     Let the corresponding sentence meanings be given by: 
\begin{align*}
\semantics{s_1} &= (\ep_N \ot 1_S \ot \ep_N)\left( \semantics{\lang{Gretel}} \ot \semantics{\lang{likes}} \ot \semantics{\lang{sweets}} \right) \\
\semantics{s_2} & = (\ep_N \ot 1_S \ot \ep_N) \left(\semantics{\lang{Gretel}} \ot \semantics{\lang{likes}} \ot \semantics{\lang{gingerbread}} \right) 
\end{align*}
Observe that:
\[
\semantics{\lang{gingerbread}} \hypo_k \semantics{\lang{sweets}} \hspace{0.3cm} \text{for} \hspace{0.2cm} k \leq \frac{1}{10} .
\]
In particular, we have~$k_{max}$-hyponymy between \lang{gingerbread} and \lang{sweets} for~$k_{max} = \frac{1}{10}$.
We will now show that this hyponymy translates to the sentence level. With~$\varphi = \ep_N \ot 1_S \ot \ep_N$ and~$\rho = \sum_{i=3}^{m} p_i \ket{n_i}\bra{n_i}$ we have:
\begin{align*}
  \semantics{s_1} &= \varphi\left( \ket{n_1}\bra{n_1} \: \ot \semantics{\lang{likes}}  \ot \left( \frac{1}{10} \,\ket{n_2} \bra{n_2} + \rho  \right) \right) \\
  & = \frac{1}{10}\, \varphi\left( \ket{n_1}\bra{n_1} \: \ot \semantics{\lang{likes}} \ot \: \ket{n_2} \bra{n_2} \right) \\
  & \qquad + \varphi \left( \ket{n_1}\bra{n_1} \: \ot \semantics{\lang{likes}} \ot \rho \right) \\
  \semantics{s_2} & =  \varphi \left( \ket{n_1}\bra{n_1}  \ot \semantics{\lang{likes}} \ot \ket{n_2} \bra{n_2} \right) 
\end{align*} 
We claim that the maximum $k$-hyponymy between~$\semantics{s_2}$ and~$\semantics{s_1}$ is achieved for~$k = \frac{1}{10}$.
In other words, this is the maximum value of~$k$ for which we have~$\semantics{s_2} \hypo_p \semantics{s_1}$, i.e. $\semantics{s_1} - p \, \semantics{s_2} \psd 0$.
We first show that~$\semantics{s_1} - \frac{1}{10}\semantics{s_2}$ is positive.
\begin{align*} \label{eq:2}
  \semantics{s_1} - \frac{1}{10} \semantics{s_2} &= \varphi \left( \ket{n_1}\bra{n_1} \: \ot \semantics{\lang{likes}} \ot \rho \right) \\
  & = (\ep_N \ot 1_S \ot \ep_N) \left( \ket{n_1}\bra{n_1} \: \ot \semantics{\lang{likes}} \ot \rho \right)\\
  &= \sum_{i=3}^{m}C_{1i1i} p_i \, \ket{n_1}\bra{n_1} \ot \ket{n_i}\bra{n_i}      
\end{align*}
This is positive by positivity of~$C_{1i1i}$ and~$p_i$.

For a value of~$k = \frac{1}{10} + \epsilon$, by a similar calculation we obtain:
\[
\varphi(\semantics{s_1} - k\semantics{s_2}) = \sum_{i=3}^{m} C_{1i1i} p_i \ket{x_{1i}}\bra{x_{1i}}  - \epsilon\ket{x_{12}}\bra{x_{12}}
\]
We then note that:
\[
\bra{x_{12}}(\varphi(\semantics{A} - k\semantics{B}))\ket{x_{12}} = -\epsilon
\]
and therefore~$k = \frac{1}{10}$ is maximal.

\subsubsection{$k$-hyponymy for positive transitive sentences}

Now suppose that the subject and object vector spaces are two-dimensional with bases~$\ket{e_1}, \ket{e_2}$
and~$\ket{n_1}, \ket{n_2}$ respectively. We let: 
\[
\semantics{\lang{Hansel}} = \ket{e_1}\bra{e_1}, \hspace{0.5cm} \semantics{Gretel} = \ket{e_2}\bra{e_2}
\]
\[
\semantics{\lang{gingerbread}} = \ket{n_1}\bra{n_1} , \hspace{0.5cm} \semantics{cake} = \ket{n_2}\bra{n_2}
\]
The density matrices for the hypernyms \lang{the siblings} and \lang{sweets} are: 
\[
\semantics{\lang{the siblings}} = \frac{1}{2}\, \semantics{\lang{Hansel}} + \frac{1}{2} \, \semantics{\lang{Gretel}}
\]
\[
\semantics{\lang{sweets}} = \frac{1}{2}\, \semantics{\lang{gingerbread}} + \frac{1}{2} \, \semantics{\lang{cake}} .
\]     
The verb \lang{like} is given as before and we assume that Gretel likes gingerbread but not cake and Hansel likes both.
The sentence reduction map~$\varphi$ is again~$(\ep_N \ot 1_S \ot \ep_N)$.  Then we have: 
\begin{align*}
  &\semantics{s_1} = \varphi\left( \semantics{\lang{Gretel}} \ot \semantics{\lang{likes}} \ot \semantics{\lang{gingerbread}} \right) \\
  &\semantics{s_2} = \varphi \left( \semantics{\lang{the siblings}} \ot \semantics{\lang{like}} \ot \semantics{\lang{sweets}} \right) \\
  & =\frac{1}{4} \, \varphi \left(\semantics{\lang{Gretel}} \ot \semantics{\lang{likes}} \ot \semantics{\lang{gingerbread}} \right)\\
  &\qquad + \frac{1}{4} \,\varphi \left( \semantics{\lang{Gretel}} \ot \semantics{\lang{likes}} \ot \semantics{\lang{cake}} \right) \\ 
  & \qquad + \frac{1}{4} \, \varphi \left( \semantics{\lang{Hansel}} \ot \semantics{\lang{likes}} \ot \left(\semantics{\lang{gingerbread}} + \semantics{\lang{cake}} \right) \right) 
\end{align*}
We then have:
\begin{align*}
  & \semantics{s_2} - \frac{1}{4} \semantics{s_1} =  \frac{1}{4} \varphi \left( \semantics{\lang{Gretel}} \ot \semantics{\lang{likes}} \ot \semantics{\lang{cake}} \right) \\
  & \quad + \left( \semantics{\lang{Hansel}} \ot \semantics{\lang{likes}} \ot \left(\semantics{\lang{gingerbread}} + \semantics{\lang{cake}} \right) \right) \\
  &= \frac{1}{4}(\ket{x_{22}} \bra{x_{22}} + \ket{x_{11}}\bra{x_{11}} + \ket{x_{12}}\bra{x_{12}})
\end{align*}
which is clearly positive.

Again, $k = \frac{1}{4}$ is maximal, since taking~$k' = \frac{1}{4} + \epsilon$ gives us the following:
\begin{align*}
  & \semantics{s_2} - k'\semantics{s_1}  =  \frac{1}{4} \varphi \left( \semantics{\lang{Gretel}} \ot \semantics{\lang{like}} \ot \semantics{\lang{cake}} \right) \\
  & + \left( \semantics{\lang{Hansel}} \ot \semantics{\lang{like}} \ot \left(\semantics{\lang{gingerbread}} + \semantics{\lang{cake}} \right) \right) \\
  & - \epsilon \left(\semantics{\lang{Gretel}} \ot \semantics{\lang{like}} \ot \semantics{\lang{gingerbread}} \right)\\
  & = \frac{1}{4}(\ket{x_{22}} \bra{x_{22}} + \ket{x_{11}}\bra{x_{11}} + \ket{x_{12}}\bra{x_{12}}) \\
  & \qquad - \epsilon (\ket{x_{21}}\bra{x_{21}})
\end{align*}
Then~$\semantics{s_2} - k'\semantics{s_1}$ is no longer positive, since:
\[
\bra{x_{21}}(\semantics{B} - k'\semantics{A})\ket{x_{21}} = -\epsilon
\]
and therefore~$\frac{1}{4}$ is maximal.

In these last two examples, the value of~$k$ that transfers to the sentence space is maximal. In general this will not be the case.
The reason that the maximality of the~$k$ transfers in these examples is due to the orthogonality of the noun vectors that we work with.

\section{Applying the theory to Frobenius Algebras}
\label{app:frobenius}
This appendix details techniques that we have not included in the main body of the text.
\subsection{Frobenius Algebras}
\label{sec:frob}
We state here how a Frobenius algebra is implemented within a vector space over~$\mathbb{R}$.
For a mathematically rigorous presentation see~\cite{sadrzadeh2013}.
A real vector space with a fixed basis~$\{\ket{v_i}\}_i$ has a Frobenius algebra given by:
\[
\Delta::\ket{v_i} \mapsto \ket{v_i} \otimes \ket{v_i} \quad \iota :: \ket{v_i} \mapsto 1 
\]
\[
\mu :: \ket{v_i} \otimes \ket{v_i} \mapsto \delta_{ij} \ket{v_i} \quad \zeta :: 1 \mapsto \sum_i \ket{v_i}
\]
This algebra is commutative, so for the swap map~$\sigma: X \otimes Y \rightarrow Y\otimes X$, we have~$\sigma \circ \Delta = \Delta$ and~$\mu \circ \sigma = \mu$.
It is also special so that~$\mu \circ \Delta = 1$. Essentially, the~$\mu$ morphism amounts to taking the diagonal of a matrix,
and~$\Delta$ to embedding a vector within a diagonal matrix. This algebra may be used to model the flow of information in noun phrases with relative pronouns.

\subsubsection{An example noun phrase}
In~\citet{sadrzadeh2013}, the authors describe how the subject and object relative pronouns may be analyzed.
We describe here the subject relative pronoun. The phrase \lang{John who kicks cats} is a noun phrase; it describes John.
The meaning of the phrase should therefore be \lang{John}, modified somehow.
The word~\lang{who} is typed~$n^r n s^l n$, so the sentence~\lang{John who kicks cats} may be reduced as follows:

\begin{figure}[htbp]
\centering
\begin{tikzpicture}[text height=1.5 ex]
	\begin{pgfonlayer}{nodelayer}
		\node [style=none] (0) at (-3.5, 1.75) {John};
		\node [style=none] (1) at (-1.5, 1.75) {who};
		\node [style=none] (2) at (1, 1.75) {kicks};
		\node [style=none] (3) at (2.75, 1.75) {cats};
		\node [style=none] (4) at (-3.5, 0.75) {$n$};
		\node [style=none] (5) at (-2.25, 0.75) {$n^r$};
		\node [style=none] (6) at (-1.75, 0.75) {$n$};
		\node [style=none] (7) at (-1.25, 0.75) {$s^l$};
		\node [style=none] (8) at (-0.75, 0.75) {$n$};
		\node [style=none] (9) at (0.5, 0.75) {$n^r$};
		\node [style=none] (10) at (1, 0.75) {$s$};
		\node [style=none] (11) at (1.5, 0.75) {$n^l$};
		\node [style=none] (12) at (2.75, 0.75) {$n$};
		\node [style=none] (13) at (-3.5, 0.25) {};
		\node [style=none] (14) at (-2.25, 0.25) {};
		\node [style=none] (15) at (-1.75, 0.25) {};
		\node [style=none] (16) at (-1.25, 0.25) {};
		\node [style=none] (17) at (-0.75, 0.25) {};
		\node [style=none] (18) at (0.5, 0.25) {};
		\node [style=none] (19) at (1, 0.25) {};
		\node [style=none] (20) at (1.5, 0.25) {};
		\node [style=none] (21) at (2.75, 0.25) {};
		\node [style=none] (22) at (-1.75, -1) {};
	\end{pgfonlayer}
	\begin{pgfonlayer}{edgelayer}
		\draw [thick, bend right=90] (13.center) to (14.center);
		\draw [thick, bend right=90, looseness=1.25] (17.center) to (18.center);
		\draw [thick, bend right=90, looseness=1.25] (16.center) to (19.center);
		\draw [thick, bend left=90] (21.center) to (20.center);
		\draw [thick] (15.center) to (22.center);
	\end{pgfonlayer}
\end{tikzpicture}
\end{figure}
\citet{sadrzadeh2013} analyse the subject-relative pronoun \lang{who} as having a structure that can be formalised using the Frobenius algebra as follows:

\begin{figure}[h!]
\centering
\begin{tikzpicture}
	\begin{pgfonlayer}{nodelayer}
		\node [style=none] (0) at (-3.5, 3) {John};
		\node [style=none] (1) at (-1.5, 3) {who};
		\node [style=none] (2) at (1, 3) {kicks};
		\node [style=none] (3) at (3, 3) {cats};
		\node [style=blank] (4) at (-3.5, 1.5) {$N$};
		\node [style=blank] (5) at (-2.5, 1.5) {$N$};
		\node [style=blank] (6) at (-1.5, 1.25) {$N$};
		\node [style=blank] (7) at (-0.5, 1.5) {$N$};
		\node [style=blank] (8) at (0.5, 1.5) {$N$};
		\node [style=blank] (9) at (1, 1.5) {$S$};
		\node [style=blank] (10) at (1.5, 1.5) {$N$};
		\node [style=blank] (11) at (3, 1.5) {$N$};
		\node [style=small_node] (12) at (-1.5, 1.75) {};
		\node [style=none] (13) at (-2.5, 2.25) {};
		\node [style=none] (14) at (-2, 2.25) {};
		\node [style=none] (15) at (-1, 2.25) {};
		\node [style=none] (16) at (-0.5, 2.25) {};
		\node [style=small_node] (17) at (1, 1) {};
		\node [style=none] (18) at (-4, 2) {};
		\node [style=none] (19) at (-3, 2) {};
		\node [style=none] (20) at (-3.5, 2.5) {};
		\node [style=none] (21) at (0, 2) {};
		\node [style=none] (22) at (1, 2.75) {};
		\node [style=none] (23) at (2, 2) {};
		\node [style=none] (24) at (2.5, 2) {};
		\node [style=none] (25) at (3, 2.5) {};
		\node [style=none] (26) at (3.5, 2) {};
		\node [style=none] (27) at (-3.5, 2) {};
		\node [style=none] (28) at (0.5, 2) {};
		\node [style=none] (29) at (1, 2) {};
		\node [style=none] (30) at (1.5, 2) {};
		\node [style=none] (31) at (3, 2) {};
		\node [style=none] (32) at (-1.5, 0.75) {};
	\end{pgfonlayer}
	\begin{pgfonlayer}{edgelayer}
		\draw [thick, bend left=90, looseness=2.00] (13.center) to (14.center);
		\draw [thick, in=180, out=-90, looseness=1.25] (14.center) to (12);
		\draw [thick, in=0, out=-90, looseness=1.25] (15.center) to (12);
		\draw [thick, bend left=90, looseness=1.75] (15.center) to (16.center);
		\draw [thick] (16.center) to (7);
		\draw [thick] (12) to (6);
		\draw [thick, bend right=90, looseness=1.25] (4) to (5);
		\draw [thick, bend right=90, looseness=1.25] (7) to (8);
		\draw [thick] (9) to (17);
		\draw [thick, bend right=90] (10) to (11);
		\draw [thick] (18.center) to (19.center);
		\draw [thick] (20.center) to (18.center);
		\draw [thick] (20.center) to (19.center);
		\draw [thick] (21.center) to (23.center);
		\draw [thick] (21.center) to (22.center);
		\draw [thick] (22.center) to (23.center);
		\draw [thick] (25.center) to (24.center);
		\draw [thick] (24.center) to (26.center);
		\draw [thick] (25.center) to (26.center);
		\draw [thick] (13.center) to (5);
		\draw [thick] (27.center) to (4);
		\draw [thick] (28.center) to (8);
		\draw [thick] (29.center) to (9);
		\draw [thick] (30.center) to (10);
		\draw [thick] (31.center) to (11);
		\draw [thick] (6) to (32.center);
	\end{pgfonlayer}
\end{tikzpicture}
\end{figure}Straightening wires allows us to simplify this to the following:

\begin{figure}[h!]
\centering
\begin{tikzpicture}
	\begin{pgfonlayer}{nodelayer}
		\node [style=none] (0) at (-1.75, 1.75) {John};
		\node [style=none] (1) at (0.25, 1.75) {kicks};
		\node [style=none] (2) at (2.25, 1.75) {cats};
		\node [style=blank] (3) at (-1.75, 0.25) {$N$};
		\node [style=blank] (4) at (-0.25, 0.25) {$N$};
		\node [style=blank] (5) at (0.25, 0.25) {$S$};
		\node [style=blank] (6) at (0.75, 0.25) {$N$};
		\node [style=blank] (7) at (2.25, 0.25) {$N$};
		\node [style=small_node] (8) at (0.25, -0.5) {};
		\node [style=none] (9) at (-2.25, 0.75) {};
		\node [style=none] (10) at (-1.25, 0.75) {};
		\node [style=none] (11) at (-1.75, 1.25) {};
		\node [style=none] (12) at (-0.75, 0.75) {};
		\node [style=none] (13) at (0.25, 1.5) {};
		\node [style=none] (14) at (1.25, 0.75) {};
		\node [style=none] (15) at (1.75, 0.75) {};
		\node [style=none] (16) at (2.25, 1.25) {};
		\node [style=none] (17) at (2.75, 0.75) {};
		\node [style=none] (18) at (-1.75, 0.75) {};
		\node [style=none] (19) at (-0.25, 0.75) {};
		\node [style=none] (20) at (0.25, 0.75) {};
		\node [style=none] (21) at (0.75, 0.75) {};
		\node [style=none] (22) at (2.25, 0.75) {};
		\node [style=small_node] (23) at (-1, -0.25) {};
		\node [style=none] (24) at (-1, -1) {};
	\end{pgfonlayer}
	\begin{pgfonlayer}{edgelayer}
		\draw [thick] (5) to (8);
		\draw [thick, bend right=90] (6) to (7);
		\draw [thick] (9.center) to (10.center);
		\draw [thick] (11.center) to (9.center);
		\draw [thick] (11.center) to (10.center);
		\draw [thick] (12.center) to (14.center);
		\draw [thick] (12.center) to (13.center);
		\draw [thick] (13.center) to (14.center);
		\draw [thick] (16.center) to (15.center);
		\draw [thick] (15.center) to (17.center);
		\draw [thick] (16.center) to (17.center);
		\draw [thick] (18.center) to (3);
		\draw [thick] (19.center) to (4);
		\draw [thick] (20.center) to (5);
		\draw [thick] (21.center) to (6);
		\draw [thick] (22.center) to (7);
		\draw [thick, bend right=45] (3) to (23);
		\draw [thick, bend left=45] (4) to (23);
		\draw [thick] (23) to (24.center);
	\end{pgfonlayer}
\end{tikzpicture}
\end{figure}

\begin{table}
\centering
\caption{Table of diagrams for Frobenius algebras in $\CPMC$ and $\mathcal{C}$}
\label{tab:frob} 
\begin{tabular}{ c c}
  $\CPMC$ & $\mathcal{C}$\\
  \hline
    $E(\mu) = \mu_* \otimes \mu$ & $\mu : A^* \otimes A \otimes A^* \otimes A \rightarrow A^* \otimes A$ \\
  \begin{tikzpicture}
	\begin{pgfonlayer}{nodelayer}
		\node [style=none] (0) at (-1, 0.25) {$A$};
		\node [style=none] (1) at (1, 0.25) {$A$};
		\node [style=none] (2) at (-1, 0) {};
		\node [style=none] (3) at (1, 0) {};
		\node [style=small circ] (4) at (0, -1) {};
		\node [style=none] (5) at (0, -2) {};
		\node [style=none] (6) at (0, -2.25) {$A$};
	\end{pgfonlayer}
	\begin{pgfonlayer}{edgelayer}
		\draw [ultra thick, bend left=45, looseness=1.25] (4) to (2.center);
		\draw [ultra thick, bend right=45, looseness=1.25] (4) to (3.center);
		\draw [ultra thick] (5.center) to (4);
	\end{pgfonlayer}
\end{tikzpicture}&\begin{tikzpicture}
	\begin{pgfonlayer}{nodelayer}
		\node [style=none] (0) at (-1, 0.25) {$A^*$};
		\node [style=none] (1) at (1, 0.25) {$A^*$};
		\node [style=none] (2) at (-1, 0) {};
		\node [style=none] (3) at (1, 0) {};
		\node [style=small circ] (4) at (0, -1) {};
		\node [style=none] (5) at (0, -2) {};
		\node [style=none] (6) at (0, -2.25) {$A^*$};
		\node [style=small circ] (7) at (1.5, -1) {};
		\node [style=none] (8) at (2.5, 0) {};
		\node [style=none] (9) at (1.5, -2.25) {$A$};
		\node [style=none] (10) at (0.5, 0.25) {$A$};
		\node [style=none] (11) at (0.5, 0) {};
		\node [style=none] (12) at (1.5, -2) {};
		\node [style=none] (13) at (2.5, 0.25) {$A$};
	\end{pgfonlayer}
	\begin{pgfonlayer}{edgelayer}
		\draw [thick, bend left=45, looseness=1.25] (4) to (2.center);
		\draw [thick, bend right=45, looseness=1.25] (4) to (3.center);
		\draw [thick] (5.center) to (4);
		\draw [thick, bend left=45, looseness=1.25] (7) to (11.center);
		\draw [thick, bend right=45, looseness=1.25] (7) to (8.center);
		\draw [thick] (12.center) to (7);
	\end{pgfonlayer}
\end{tikzpicture}\\
  \multicolumn{2}{c}{$\mu: \ket{e_i} \otimes \ket{e_j} \otimes \ket{e_k} \otimes \ket{e_l} \mapsto \braket{e_i| e_k}\braket{ e_j| e_l}(\ket{e_i} \otimes \ket{e_j})$} \\
$E(\Delta) = \Delta_* \otimes \Delta$ &  $\Delta : A^* \otimes A  \rightarrow A^* \otimes A \otimes A^* \otimes A $\\
\begin{tikzpicture}
	\begin{pgfonlayer}{nodelayer}
		\node [style=none] (0) at (-1, -2.25) {$A$};
		\node [style=none] (1) at (1, -2.25) {$A$};
		\node [style=none] (2) at (-1, -2) {};
		\node [style=none] (3) at (1, -2) {};
		\node [style=small circ] (4) at (0, -1) {};
		\node [style=none] (5) at (0, 0) {};
		\node [style=none] (6) at (0, 0.25) {$A$};
	\end{pgfonlayer}
	\begin{pgfonlayer}{edgelayer}
		\draw [ultra thick, bend right=45, looseness=1.25] (4) to (2.center);
		\draw [ultra thick, bend left=45, looseness=1.25] (4) to (3.center);
		\draw [ultra thick] (5.center) to (4);
	\end{pgfonlayer}
\end{tikzpicture}&\begin{tikzpicture}
	\begin{pgfonlayer}{nodelayer}
		\node [style=none] (0) at (-1, -2.25) {$A^*$};
		\node [style=none] (1) at (1, -2.25) {$A^*$};
		\node [style=none] (2) at (-1, -2) {};
		\node [style=none] (3) at (1, -2) {};
		\node [style=small circ] (4) at (0, -1) {};
		\node [style=none] (5) at (0, 0) {};
		\node [style=none] (6) at (0, 0.25) {$A^*$};
		\node [style=small circ] (7) at (1.5, -1) {};
		\node [style=none] (8) at (2.5, -2) {};
		\node [style=none] (9) at (1.5, 0.25) {$A$};
		\node [style=none] (10) at (0.5, -2.25) {$A$};
		\node [style=none] (11) at (0.5, -2) {};
		\node [style=none] (12) at (1.5, 0) {};
		\node [style=none] (13) at (2.5, -2.25) {$A$};
	\end{pgfonlayer}
	\begin{pgfonlayer}{edgelayer}
		\draw [thick, bend right=45, looseness=1.25] (4) to (2.center);
		\draw [thick, bend left=45, looseness=1.25] (4) to (3.center);
		\draw [thick] (5.center) to (4);
		\draw [thick, bend right=45, looseness=1.25] (7) to (11.center);
		\draw [thick, bend left=45, looseness=1.25] (7) to (8.center);
		\draw [thick] (12.center) to (7);
	\end{pgfonlayer}
\end{tikzpicture}\\
\multicolumn{2}{c}{$\Delta: \ket{e_i} \otimes \ket{e_j} \mapsto \sum_{ij}\ket{e_i} \otimes \ket{e_j} \otimes \ket{e_i} \otimes \ket{e_j} $} \\
$E(\iota) = \iota_* \otimes \iota$ & $\iota : I \rightarrow A^* \otimes A$ \\
  \begin{tikzpicture}
	\begin{pgfonlayer}{nodelayer}
		\node [style=small circ] (0) at (0, -1) {};
		\node [style=none] (1) at (0, 0) {};
		\node [style=none] (2) at (0, 0.25) {$A$};
	\end{pgfonlayer}
	\begin{pgfonlayer}{edgelayer}
		\draw [ultra thick] (1.center) to (0);
	\end{pgfonlayer}
\end{tikzpicture}&\begin{tikzpicture}
	\begin{pgfonlayer}{nodelayer}
		\node [style=small circ] (0) at (0, -1) {};
		\node [style=none] (1) at (0, 0) {};
		\node [style=none] (2) at (0, 0.25) {$A^*$};
		\node [style=small circ] (3) at (1.5, -1) {};
		\node [style=none] (4) at (1.5, 0.25) {$A$};
		\node [style=none] (5) at (1.5, 0) {};
	\end{pgfonlayer}
	\begin{pgfonlayer}{edgelayer}
		\draw [thick] (1.center) to (0);
		\draw [thick] (5.center) to (3);
	\end{pgfonlayer}
\end{tikzpicture}\\
  \multicolumn{2}{c}{$\iota: \ket{e_i} \otimes \ket{e_j} \mapsto 1$} \\
$E(\zeta) = \zeta_* \otimes \zeta$ &  $\zeta : A^* \otimes A  \rightarrow I $\\ 
\begin{tikzpicture}
	\begin{pgfonlayer}{nodelayer}
		\node [style=small circ] (0) at (0, 0.25) {};
		\node [style=none] (1) at (0, -0.75) {};
		\node [style=none] (2) at (0, -1) {$A$};
	\end{pgfonlayer}
	\begin{pgfonlayer}{edgelayer}
		\draw [ultra thick] (1.center) to (0);
	\end{pgfonlayer}
\end{tikzpicture}&\begin{tikzpicture}
	\begin{pgfonlayer}{nodelayer}
		\node [style=small circ] (0) at (0, 0.25) {};
		\node [style=none] (1) at (0, -0.75) {};
		\node [style=none] (2) at (0, -1) {$A^*$};
		\node [style=small circ] (3) at (1.5, 0.25) {};
		\node [style=none] (4) at (1.5, -1) {$A$};
		\node [style=none] (5) at (1.5, -0.75) {};
	\end{pgfonlayer}
	\begin{pgfonlayer}{edgelayer}
		\draw [thick] (1.center) to (0);
		\draw [thick] (5.center) to (3);
	\end{pgfonlayer}
\end{tikzpicture}\\
\multicolumn{2}{c}{$\zeta: 1 \mapsto \sum_{i}\ket{e_1} \otimes \ket{e_i}$} \\ 
  \hline
\end{tabular}
\end{table}

\subsection{$k$-hyponymy in relative clauses}
\label{sec:krelpron}
Relative clauses are expressions such as~\lang{John who kicks cats}.
These are noun phrases, and the diagrammatic representation of such phrases was introduced in section~\ref{sec:frob}.
As for sentences, the diagram in~$\CPM{\FHilb}$ is equivalent to the diagram in~$\FHilb$ but with thick wires, given in figure \ref{fig:frob-sub-cpmc}
The diagrammatic representation of subject relative clauses in~$\FHilb$ is given in figure \ref{fig:frob-sub-c}
\begin{figure}[htbp]
\centering
\begin{tikzpicture}
	\begin{pgfonlayer}{nodelayer}
		\node [style=none] (0) at (-1, 3) {subject};
		\node [style=none] (1) at (1.5, 3) {verb};
		\node [style=none] (2) at (4, 3) {object};
		\node [style=none] (3) at (-1, 2) {};
		\node [style=none] (4) at (0, 0.5) {};
		\node [style=none] (5) at (1, 2) {};
		\node [style=none] (6) at (1.5, 2) {};
		\node [style=none] (7) at (2, 2) {};
		\node [style=none] (8) at (4, 2) {};
		\node [style=small_node] (9) at (0, 1.25) {};
		\node [style=small_node] (10) at (1.5, 1) {};
		\node [style=none] (11) at (-1.5, 2) {};
		\node [style=none] (12) at (-0.5, 2) {};
		\node [style=none] (13) at (-1, 2.5) {};
		\node [style=none] (14) at (0.5, 2) {};
		\node [style=none] (15) at (1.5, 2.75) {};
		\node [style=none] (16) at (2.5, 2) {};
		\node [style=none] (17) at (3.5, 2) {};
		\node [style=none] (18) at (4, 2.5) {};
		\node [style=none] (19) at (4.5, 2) {};
	\end{pgfonlayer}
	\begin{pgfonlayer}{edgelayer}
		\draw [ultra thick] (9) to (4.center);
		\draw [ultra thick] (6.center) to (10);
		\draw [ultra thick, bend right=90, looseness=1.25] (7.center) to (8.center);
		\draw [ultra thick] (11.center) to (12.center);
		\draw [ultra thick] (13.center) to (11.center);
		\draw [ultra thick] (13.center) to (12.center);
		\draw [ultra thick] (14.center) to (16.center);
		\draw [ultra thick] (14.center) to (15.center);
		\draw [ultra thick] (15.center) to (16.center);
		\draw [ultra thick] (18.center) to (17.center);
		\draw [ultra thick] (17.center) to (19.center);
		\draw [ultra thick] (18.center) to (19.center);
		\draw [ultra thick, bend right=45] (3.center) to (9);
		\draw [ultra thick, bend left=45] (5.center) to (9);
	\end{pgfonlayer}
\end{tikzpicture}
\caption{A noun phrase generated by the subject relative pronoun sentence in $\CPMC$}
\label{fig:frob-sub-cpmc}
\end{figure}
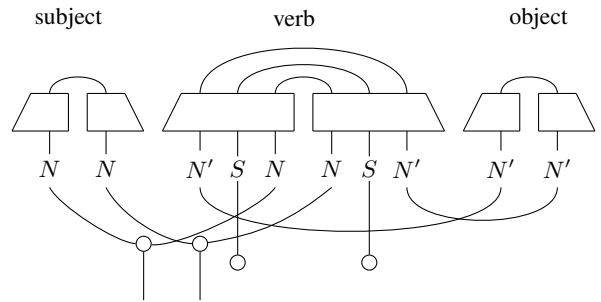
\begin{figure}[htbp]
\centering
\begin{tikzpicture}
	\begin{pgfonlayer}{nodelayer}
		\node [style=none] (0) at (0.25, -0.75) {};
		\node [style=none] (1) at (0.75, 0) {};
		\node [style=none] (2) at (1.25, -0.75) {};
		\node [style=none] (3) at (3.25, -0.75) {};
		\node [style=none] (4) at (3.75, 0) {};
		\node [style=none] (5) at (3, 0) {};
		\node [style=none] (6) at (3, 0.5) {};
		\node [style=none] (7) at (3.5, 0.5) {};
		\node [style=none] (8) at (3.25, 0.5) {};
		\node [style=none] (9) at (2.75, 0.5) {};
		\node [style=none] (10) at (2, 0) {};
		\node [style=none] (11) at (2.25, 0.5) {};
		\node [style=none] (12) at (2.5, 0.5) {};
		\node [style=none] (13) at (2.5, -0.75) {};
		\node [style=none] (14) at (2.75, 0) {};
		\node [style=none] (15) at (1.75, 0) {};
		\node [style=none] (16) at (1.5, 0.5) {};
		\node [style=none] (17) at (0, 0.5) {};
		\node [style=none] (18) at (0, 0) {};
		\node [style=blank] (19) at (0.25, -0.5) {$N$};
		\node [style=blank] (20) at (0.75, -0.5) {$S$};
		\node [style=blank] (21) at (1.25, -0.5) {$N'$};
		\node [style=blank] (22) at (2.5, -0.5) {$N'$};
		\node [style=blank] (23) at (3.25, -0.5) {$N'$};
		\node [style=none] (24) at (0.25, 0.5) {};
		\node [style=none] (25) at (0.75, 0.5) {};
		\node [style=none] (26) at (1.25, 0.5) {};
		\node [style=small_node] (27) at (0.75, -1.75) {};
		\node [style=none] (28) at (-1, 0) {};
		\node [style=none] (29) at (-3, 0.5) {};
		\node [style=none] (30) at (-1.5, 0.5) {};
		\node [style=blank] (31) at (-0.5, -0.5) {$N$};
		\node [style=blank] (32) at (-2.75, -0.5) {$N$};
		\node [style=none] (33) at (-1.75, 0.5) {};
		\node [style=none] (34) at (-0.5, -0.75) {};
		\node [style=none] (35) at (-3.25, 0) {};
		\node [style=none] (36) at (-2, 0) {};
		\node [style=blank] (37) at (-3.5, -0.5) {$N$};
		\node [style=none] (38) at (-2.75, -0.75) {};
		\node [style=none] (39) at (-2.25, 0) {};
		\node [style=none] (40) at (-4, 0) {};
		\node [style=none] (41) at (-3.5, 0.5) {};
		\node [style=none] (42) at (-0.25, 0) {};
		\node [style=none] (43) at (-2.5, 0.5) {};
		\node [style=small_node] (44) at (-1, -1.75) {};
		\node [style=none] (45) at (-3.25, 0.5) {};
		\node [style=blank] (46) at (-1.5, -0.5) {$N'$};
		\node [style=none] (47) at (-1, 0.5) {};
		\node [style=none] (48) at (-0.5, 0.5) {};
		\node [style=none] (49) at (-1.5, -0.75) {};
		\node [style=none] (50) at (-2.75, 0.5) {};
		\node [style=none] (51) at (-3, 0) {};
		\node [style=none] (52) at (-3.75, 0.5) {};
		\node [style=none] (53) at (-0.25, 0.5) {};
		\node [style=blank] (54) at (-1, -0.5) {$S$};
		\node [style=none] (55) at (-3.5, -0.75) {};
		\node [style=none] (56) at (-3.5, 0) {};
		\node [style=none] (57) at (-2.75, 0) {};
		\node [style=none] (58) at (-1.5, 0) {};
		\node [style=none] (59) at (-0.5, 0) {};
		\node [style=none] (60) at (0.25, 0) {};
		\node [style=none] (61) at (1.25, 0) {};
		\node [style=none] (62) at (2.5, 0) {};
		\node [style=none] (63) at (3.25, 0) {};
		\node [style=none] (64) at (-3.25, 1.5) {subject};
		\node [style=none] (65) at (-0.25, 1.5) {verb};
		\node [style=none] (66) at (3, 1.5) {object};
		\node [style=none] (67) at (-2.25, -2.25) {};
		\node [style=none] (68) at (-1.5, -2.25) {};
		\node [style=small_node] (69) at (-2.25, -1.5) {};
		\node [style=small_node] (70) at (-1.5, -1.5) {};
	\end{pgfonlayer}
	\begin{pgfonlayer}{edgelayer}
		\draw (52.center) to (40.center);
		\draw (40.center) to (35.center);
		\draw (35.center) to (45.center);
		\draw (45.center) to (52.center);
		\draw (29.center) to (51.center);
		\draw (51.center) to (39.center);
		\draw (39.center) to (43.center);
		\draw (43.center) to (29.center);
		\draw (33.center) to (36.center);
		\draw (36.center) to (42.center);
		\draw (42.center) to (53.center);
		\draw (53.center) to (33.center);
		\draw (17.center) to (18.center);
		\draw (18.center) to (15.center);
		\draw (15.center) to (16.center);
		\draw (16.center) to (17.center);
		\draw (54) to (44);
		\draw (20) to (27);
		\draw (11.center) to (10.center);
		\draw (9.center) to (14.center);
		\draw (10.center) to (14.center);
		\draw (9.center) to (11.center);
		\draw (6.center) to (5.center);
		\draw (5.center) to (4.center);
		\draw (4.center) to (7.center);
		\draw (7.center) to (6.center);
		\draw (28.center) to (54);
		\draw (1.center) to (20);
		\draw [bend left=90] (41.center) to (50.center);
		\draw [bend left=90] (48.center) to (24.center);
		\draw [bend left=90, looseness=0.75] (47.center) to (25.center);
		\draw [bend left=90, looseness=0.75] (30.center) to (26.center);
		\draw [bend right=90] (8.center) to (12.center);
		\draw (56.center) to (37);
		\draw (57.center) to (32);
		\draw (58.center) to (46);
		\draw (59.center) to (31);
		\draw (60.center) to (19);
		\draw (61.center) to (21);
		\draw (62.center) to (22);
		\draw (63.center) to (23);
		\draw [bend left=90, looseness=0.75] (3.center) to (2.center);
		\draw [bend left=90, looseness=0.50] (13.center) to (49.center);
		\draw [bend right, looseness=0.50] (55.center) to (69);
		\draw [bend right, looseness=0.50] (69) to (34.center);
		\draw [bend right, looseness=0.75] (38.center) to (70);
		\draw [bend right=15] (70) to (0.center);
		\draw (69) to (67.center);
		\draw (70) to (68.center);
	\end{pgfonlayer}
\end{tikzpicture}
\caption{A noun phrase generated by the subject relative pronoun sentence in $\mathcal{C}$}
\label{fig:frob-sub-c}
\end{figure}

We assume that the relative pronoun is \lang{which}.
Then the meaning map for the relative clause \lang{subj which verb obj}  in~$\CPM{\FHilb}$ is~$\mu_N \ot \iota_S \ot \ep_N$ and the meaning of the relative clause is given by: 
\[
(\mu_N \ot \iota_S \ot \ep_N)(\semantics{\lang{subj}} \ot \semantics{\lang{verb}} \ot \semantics{\lang{obj}}) .
\]
We can now characterise the relationship between relative clauses `\lang{A which verb C}' and `\lang{B which verb D}' where $\semantics{A} \hypo_k \semantics{B}$
and $\semantics{C} \hypo_l \semantics{D}$, and obtain a result very similar to the one we had for the positive semi-definite sentence types, under the same assumptions. 
\begin{thm}
  Let~$n_1, n_2, n_3, n_4$ be nouns with corresponding density matrix representations $\semantics{n_1}$, $\semantics{n_2}$, $\semantics{n_3}$ and $\semantics{n_4}$,
  and such that $\semantics{n_2} = k \semantics{n_1} + D$ and $\semantics{n_4} = l \semantics{n_3} + D'$ for some $k, l \in (0,1]$. Then we have that: 
\[
\varphi \left( \semantics{n_1}\otimes\semantics{verb}\otimes\semantics{n_3}\right) \hypo_{kl} \varphi \left(\semantics{n_2}\otimes\semantics{verb}\otimes\semantics{n_4}\right)
\]
\end{thm}
\begin{proof} The proof of this result is identical to that of theorem \ref{thm:trans}, except for the fact that when we consider 
\[
\varphi(\semantics{n_2}\otimes\semantics{verb}\otimes\semantics{n_4}) - kl \, \varphi(\semantics{n_1}\otimes\semantics{verb}\otimes\semantics{n_3}) 
\]
we get~$\varphi = \mu_N \ot \iota_S \ot \ep_N$ applied to 
\[
(k \semantics{n_1} \ot \semantics{Z} \ot D') + (D \ot \semantics{Z} \ot l \semantics{n_3}) + (D \ot \semantics{Z} \ot D') 
\]
instead of~$\varphi = (\ep_N \ot 1_S \ot \ep_N)$ applied to the same.
The result is, however, still a positive quantity by the property of the morphisms~$\mu_N$, $1_S$ and $\ep_N$ to map density matrices to density matrices.
Thus, we can conclude as before that:
\[
\varphi \left(\semantics{n_2}\otimes\semantics{verb}\otimes\semantics{n_4}\right) \hypo_{kl} \varphi\left( \semantics{n_1}\otimes\semantics{verb}\otimes\semantics{n_3}\right).
\]
\end{proof}

\subsubsection{$k$-hyponymy applied to relative clauses}
We will consider the containment of the sentence:
\begin{equation*}
  s_1 = \lang{Elderly ladies who own cats} 
\end{equation*}
in the sentence:
\begin{equation*}
  s_2 = \lang{Women who own animals}
\end{equation*}
First of all, let the subject and object space for the vectors corresponding to the subjects and object of our sentences be~$\mathbb{R}^2$ and~$\mathbb{R}^3$ respectively. Let:
\[
\semantics{\lang{elderly ladies}} = \ket{e_1}\bra{e_1}, \hspace{0.5cm} \semantics{\lang{young ladies}} = \ket{e_2}\bra{e_2}
\]
and the density matrix for the hypernym \lang{women} be:
\[
\semantics{\lang{women}} = \frac{1}{3} \, \semantics{\lang{elderly ladies}} + \frac{2}{3}\, \semantics{\lang{young ladies}}
\]
Similarly, let: 
\begin{align*}
  \semantics{\lang{cats}} &= \ket{n_1}\bra{n_1}\\
  \semantics{\lang{dogs}} &= \ket{n_2}\bra{n_2}\\
  \semantics{\lang{hamsters}} &= \ket{n_3}\bra{n_3}
\end{align*}
and take the density matrix for \lang{animals} to be: 
\[
\semantics{\lang{animal}} = \frac{1}{2} \semantics{\lang{cats}} + \frac{1}{4} \semantics{\lang{dogs}} + \frac{1}{4} \semantics{\lang{hamsters}}
\]
The sentence space will not matter in this case, as it gets deleted by the~$\iota_S$ morphism, so we just take it to be an unspecified~$S$. Let the verb \lang{own} be given by:
\[
\semantics{\lang{own}} = \sum_{ijkl} C_{ijkl} \, \ket{e_i} \bra{e_k} \ot \ket{s} \bra{s'} \ot \ket{n_j} \bra{n_l} .
\]
and the sentence map~$\psi$ is given by~$\mu_N \ot \iota_S \ot \ep_N$
Then the meaning of sentences $s_1$ and $s_2$ are given by: 
\begin{align*}
\semantics{s_1} & = \psi (\semantics{\lang{elderly ladies}} \ot \semantics{\lang{own}} \ot \semantics{\lang{cats}})\\
\semantics{s_2} & = \psi(\semantics{women} \ot \semantics{\lang{own}} \ot \semantics{animals}) \\
& = \frac{1}{6} \psi (\semantics{\lang{elderly ladies}} \ot \semantics{\lang{own}} \ot \semantics{\lang{cats}}) \\
& + \frac{1}{12} \psi (\semantics{\lang{elderly ladies}}\ot \semantics{\lang{own}} \ot \semantics{\lang{dogs}}\\
& + \frac{1}{12} \psi (\semantics{\lang{elderly ladies}}\ot \semantics{\lang{own}} \ot\semantics{\lang{hamsters}}))\\
& + \frac{1}{3} \psi (\semantics{\lang{young ladies}} \ot \semantics{\lang{own}} \ot \semantics{animals})
\end{align*}
Then~$\semantics{s_2} - \frac{1}{6} \, \semantics{s_1}$ is given by: 
\begin{align*}
  &\frac{1}{12} \psi(\semantics{\lang{elderly ladies}} \ot \semantics{\lang{own}} \ot \semantics{\lang{dogs}})\\
 + &\frac{1}{12} \psi(\semantics{\lang{elderly ladies}} \ot \semantics{\lang{own}} \ot\semantics{\lang{hamsters}})\\
 + &\frac{1}{3} \psi(\semantics{\lang{young ladies}} \ot \semantics{\lang{own}} \ot \semantics{\lang{animals}})
\end{align*}
which is clearly positive.

\end{document}